\documentclass[twoside,11pt]{article} 

\usepackage{url}
\usepackage{algorithmic}
\usepackage{algorithm}
\usepackage[cmex10]{amsmath}
\usepackage{amssymb,amsmath}
\usepackage{psfrag,graphics,epsfig,epic,color,subfigure}
\usepackage{graphicx} 
\usepackage{accents}

\newtheorem{Theorem}{Theorem}

\newtheorem{Lemma}{Lemma}

\newenvironment{proof}{{\em Proof:} \ \ }{\begin{flushright}$\Box$\end{flushright}}

\newcommand{\R}{\mathbb{R}}

\newcommand{\yv}{\mathbf{y}}
\newcommand{\wv}{\mathbf{w}}
\newcommand{\vv}{\mathbf{v}}

\begin{document}

\title{FADO: A Deterministic Detection/Learning Algorithm}

\author{ Kristiaan Pelckmans}
%\addr Systems and Control (SysCon)\\
%Department of Information Technology (IT) \\
%Box 337, 751 05 UPPSALA University\\
%Sweden.}

%\editor{}

\maketitle

\begin{abstract}
	This paper\footnote{See \url{http://user.it.uu.se/~kripe367/} for additional material.},\footnote{This 
	work is supported by Swedish Research Council under contract
	621-2007-6364.} proposes and studies a detection technique for adversarial scenarios (dubbed {\em deterministic detection}).
	This technique provides an alternative detection methodology in case the usual
	stochastic methods are not applicable: 
	this can be because the studied phenomenon does not follow a stochastic sampling scheme, 
	samples are high-dimensional and subsequent multiple-testing corrections render results overly conservative,
	sample sizes are too low for asymptotic results (as e.g. the central limit theorem) to kick in,
	or one cannot allow for the small probability of failure inherent to stochastic approaches.

	This paper instead designs a method based on insights from machine learning and online learning theory:
	this detection algorithm - named Online FAult Detection (FADO) - comes with theoretical guarantees of its detection capabilities.
	A version of the {\em margin}, $\mu$, is found to regulate the detection performance of FADO.	
	A precise expression is derived for bounding the performance, and experimental results are presented assessing the influence of involved quantities. 
	A case study of scene detection is used to illustrate the approach.
	The technology is closely related to the linear perceptron rule, inherits its computational attractiveness and flexibility towards various extensions.
\end{abstract}

% Note that keywords are not normally used for peerreview papers.
%\begin{keywords}
%	Fault Detection, Online Machine Learning, Perceptron learning, Mistake-based learning, Unsupervised learning.
%\end{keywords}
%

\section{Introduction}

Detection lies at the heart of the study of inference (see e.g. the classical surveys
 \cite{liese2007statistical,van2004detection}) and of the scientific endeavour in general. 
Most theoretical results adopt explicitly a stochastic (statistical) setting: the observed samples are assumed to be sampled randomly from an underlying distribution with unknown characteristics. Deterministic alternatives often rely on a method of machine learning as clustering (see e.g the surveys
 \cite{fawcett1997adaptive, chandola2007outlier, tsai2009intrusion}). 
This paper studies an alternative unsupervised approach for doing so, 
and derives formal deterministic results relating its detection performance to a notion of {\em margin}.
In the context of supervised learning, the margin is found to be of central importance to characterise learning performance, 
see e.g. \cite{shalev2014understanding} and pointers.

To differentiate from a stochastic setting, the following nomenclature is used.
A {\em transaction} is an object of interest, playing a similar role as a sample, a record, an observation or alike.
Most of the transactions are {\em normal} ({\em non-fraudulent}), but a fraction of the collected transactions are {\em faulty} (non-normal, or {\em fraudulent}).
This property is unknown at any time, but the {\em detection algorithm} is trying to recover this information nonetheless.
A transaction is fully observed in a single {\em time-step}. At the end of the time-step,
the {\em detection algorithm} needs to make a decision about it: should it be flagged as {\em potentially faulty}, or as {\em normal}?
Hence, the setting is {\em unsupervised} and {\em online}: there is no supervision (labels, ...) to guide the detection algorithm, 
and learning proceeds time-wise.
One can solely rely on the assumption that there are more {\em normal} than {\em faulty} transactions.
That is, the faulty ones can take arbitrary form or place in the observed stream of transactions.
We ask ourself the question under what conditions one can address this formidable task successfully.
In brief, it is found that this is possible when a certain margin $\mu$ is large enough.

Many techniques were put forward in a context of machine learning which might fit this task
(see e.g. \cite{chandola2007outlier} and pointers).
A first approach to address this unsupervised task is to cluster data, and decide if subsequent samples fit this cluster solution ('normal') or not ('faulty').
A second class of methods do explicitly use the time-ordering by invoking methods of change detection.
A third class of machine learning-based approaches leans on stochastic assumptions, 
for example by estimating the support of a density using SVM-like methods 
- often going under the name anomaly detection, see e.g. \cite{scholkopf2001estimating} and subsequent work. 
Many of those techniques do not come with theoretical guarantees of the resulting {\em detection} performance.
Especially, the performance of clustering is notoriously hard to characterise per se (as found in e.g. \cite{kleinberg2003impossibility,ben2006sober,zadeh2009uniqueness,awasthi2014center}),
Another issue arises when the data does not really reflect a cluster structure (i.e. is not 'clusterable'), 
admit a parametric time-series model (e.g. is strongly time-variant), or has unbounded support.
In case such an assumption is too remote from the actual situation, a two-step approach is less appropriate.
The current approach replaces such modelling assumption by an assumption of the detectability stating that
{\em 'there exists an (approximate) detection rule of the studied form'}.
One may argue that such assumption is in many cases more natural to make.

Motivations for investigating a method for detection in a non-stochastic setup are 
\begin{itemize}
\item found in cases where transactions (faulty or normal ones) {\em do not follow a random sampling strategy}.
	Even in case of strong dependence between transactions, principles based on
	{\em exchangeability, time-varying structures, or Markov models, ...} might be used to 
	fit those observations into a stochastic framework.
	However, some situations refuse any such approach,
	especially in cases where deterministic or adversarial strategies are followed. 
	For example, consecutive frames of a movie have strong, intrinsic and varying dependendence structures.
	In case of fraud detection, cyber crime can hardly be modelled as a random source.
	Such situations are studied extensively in the context of  
	information theory (under the name of {\em  individual sequences}), in game-theory and online learning theory, 
	see \cite{cesabianchi2006} for a survey.

\item {\em Repeated application} of stochastic methods demand proper (multiple-test) corrections.
	This fact is properly acknowledged in applications of bio-informatics and medical imaging, see e.g. \cite{benjamini1995} and citing work.
	This is problematic in case of indefinite application of the detection rule: if the rule is applied in critical real-world applications,
	guarantees of the rule should not degrade in terms of the number of times it is applied.
	Also in cases where the dimensionality is overwhelming, corrections for multiple testing render (stochastic) results often overly conservative
	(see e.g. the survey by \cite{nichols2012multiple}).
	In case of deterministic detection however, no such correction is needed when performed multiple times.

\item In cases of (complex) stochastic schemes, one often invokes {\em asymptotic} results.
	These need a (relatively) large set of samples for asymptotic results to kick in.
	This in turn makes results not reliable in case of a few samples.
	In contexts of financial, medical, and other applications, this however is frequently the regime of relevance,
	and methods of machine learning are often invoked to assist.

\item Stochastic techniques necessarily admit a {\em small probability that things go wrong}.
	In some applications such (however small) probability is problematic. 
	This is especially the case for commercial systems which are going to be deployed unknown times.
	For example, one does not want to use a self-driving car that might mis-interprete critical situations.
	Similarly, one does not like to use a medical drug that might bring patients in a condition more 
	critical than the one it was designed to resolve.
\end{itemize}
One can also find ethical, governmental or societal arguments for advocating a worst-case detection strategy rather than a stochastic one.
We envision a scenario in fault detection as in \cite{bolton2002statistical} where the faults do not behave randomly
(i.e. as statistical outliers, see \cite{rousseeuw2005robust}),
but actually try to deceive the system as well as possible. Can we build a detection 
system which is insensitive to such malicious behaviour?
This hints at an application of game-theory, and indeed the used methods of online learning are closely related, see e.g. the survey \cite{cesabianchi2006}.
The object of this paper is as such not so much to device the most advanced, powerful method possible, 
but to explore a conceptual and theoretically sound alternative approach.

So the solution is set in a context of online learning.
Specifically, a setup is considered of deterministic detection of a stream of transactions, 
given as $\yv_1, \yv_2, \yv_3, \dots, \yv_t, \dots$, each represented as a vector in $\R^n$.
Algorithm (\ref{alg.ofd}) formalises the setup.

\begin{algorithm}
\caption{Online Fault Detection}
\label{alg.ofd}
\begin{algorithmic}
\STATE Initialise $\wv_0\in\R^n$ and $\epsilon_0\geq 0$.
\FOR {$t=1,2,\dots$}
\STATE(1) Receive transaction $\yv_t\in\R^n$.
\STATE(2) Decide if the transaction were {\em faulty} as
	\begin{equation}
  	\left\|\yv_t - \wv_{t-1}\right\|_2 \geq \epsilon_{t-1},
  	\end{equation}
  	or not 
\STATE(3) If so, raise an alarm. 
\STATE(4) 
	Update 
	$\wv_{t-1}\rightarrow \wv_t$ and $\epsilon_{t-1}\rightarrow\epsilon_t$.
\ENDFOR
\end{algorithmic}
\end{algorithm}

An alarm in alg. (\ref{alg.ofd}) can mean two things:
\begin{itemize}
\item Either the {\em normal} model is not adequate enough: further learning  (i.e. the alarm is {\em false}) is needed.
\item Or the transaction is really faulty (a {\em true } alarm) and the model performed adequately.
\end{itemize}
Note that flagging a transaction with {\em an alarm} does not imply that 
the transaction is necessarily stamped as {\em faulty}. 
But it indicates that there are reasons for suspicion.
Often further analysis is to be triggered in such cases,
as long as there are not too many {\em false} alarms.

This setting is an instantiation of the {\em exploration-exploitation} trade-off.
If the model were accurate enough, exploitation of the system leads to good detection and hence true alarms.
If the model were not precise enough, further {\em exploration} is needed.
The algorithm trades both behaviours so as to give good detection performance.
The key idea is to treat an alarm and consecutive learning step as a single notion called a {\em mistake},
and to bound the occurrence of such.

Note that this basic setup can be extended towards multiple situations.
For example, it requires a straightforward extension of notation to extend this results to the 
case where one has additional side-information. Another interesting extension is for the 
case where one has $K>1$ clusters. 
A third extension is to the case where lagged values of $\yv_t$ are relevant. 
While results are linear in this case, extensions using reproducing kernels (described in \cite{smola1998learning}) are immediate.
This report focusses on the {\em vanilla} (most basic) setup.

This paper is organised as follows.
The next section details the FADO algorithm and provides theoretical analysis.
Section III and IV reports numerical performances and results obtained in a case study of scene detection,
and section V concludes and provides directives for further work.

\section{FADO: Online FAult Detector}
%%%%%%%%%%%%%%%%%%%%

This section gives the formal setup.
Consider the case where the transactions are represented as a vector 
$\yv_t\in\R^n$ and where there is no side-information for this transaction available.
The {\em fraudulent} or {\em non-fraudulent} decision is encoded by a vector 
$\wv\in\R^n$ and parameter $\epsilon>0$ so that for all {\em non-fraudulent} transactions, 
one has 
\begin{equation}
	\|\wv - \yv_t\|_2  <  \epsilon,
	\label{eq.case.w}
\end{equation}
and for the {\em fraudulent} ones holds that 
\begin{equation}
	\|\wv - \yv_t\|_2 \geq  \epsilon.
	\label{eq.case.w}
\end{equation}
How to go about learning $\wv$ and $\epsilon$, while detecting?
The subsequent analysis follows this strategy:
\begin{enumerate}
\item First, we assume that {\em all} transactions are normal, and faults detected by the algorithm are necessarily {\em false}.
	This is similar to the so-called {\em realisability} assumption in learning theory, see e.g. \cite{vapnik95,shalev2014understanding}.
	Theorem 1 establish that those cannot be too many for the studied FADO algorithm.
\item Then we proof that at a time $t$, the solution given by FADO achieves a certain power of detection.
	That is, we quantify how {\em abnormal} the transaction at time $t+1$ need to be under these conditions, 
	to be  detected as {\em faulty} (positive alarm).
\item Finally, it is derived how the presence of faulty transactions is not going to influence the subsequent solution nor the bounds too much.
\end{enumerate}
This logic is followed in the derivations below.
%One can see this representation as a spacial case of $K$-means,
%where $K=1$. While this might seem reductive, the resulting problem is 
%more well defined than the NP-hard K-means, and  allows for handling complexity by using instead 
%high-dimensional representations.

\subsection{Realisable case, known $\epsilon>0$}
%%%%%%%%%%%%%%%%%%

At first, we consider the case where the radius $\epsilon>0$ is given and fixed,
and we consider how to learn $\wv$ by considering the sequence $\wv_1, \wv_2, \dots, \wv_t,  \dots\rightarrow\wv$.
Given this sequence, then the decision are made as 
\begin{equation}
	d_t=
		\begin{cases}
			1 & \mbox{ \ If \ } \| \yv_t - \wv_{t-1}\|_2\geq \epsilon\\
			0 & \mbox{ \ Otherwise.}	
		\end{cases}
		\label{eq.d2}
\end{equation}
The FADO algorithm constructs the sequence $\{\wv_t\}_t$ as detailed in Alg. (\ref{alg.fado1}).

\begin{algorithm}
\caption{FADO($\epsilon$)}
\label{alg.fado1}
\begin{algorithmic}
\STATE Initialise $\wv_0=0_d$.
\FOR {$t=1,2,\dots$}
\STATE(1) Receive transaction $\yv_t\in\R^n$.
\STATE(2) Raise an alarm if 
	\[ \left\|\yv_t - \wv_{t-1}\right\|_2 \geq \epsilon, \]
	and set 
	\[ \vv_t = \frac{ \yv_t - \wv_{t-1}}{\|\yv_t - \wv_{t-1}\|_2}\in\R^n.\]
\STATE(3) If an alarm is raised, update
	\[ 	\wv_t = \wv_{t-1} + \gamma_t \vv_t.\]
	Otherwise, set $\wv_t=\wv_{t-1}$.
\ENDFOR
\end{algorithmic}
\end{algorithm}

The number of alarms raised by this algorithm prior to time $t$, is denoted 
as $m_t$ where
\begin{equation}
	m_t = \sum_{s=1}^t d_s.
	\label{eq.m}
\end{equation}
Now we proof that this method makes only a bounded number of mistakes, 
that is $m_t<\infty$ for any $t$.
Hereto we assume the existence of a perfect solution $\bar\wv\in\R^n$ for this 
$\epsilon$ such that the following holds for all transactions,
\begin{equation}
	\| \yv_t - \bar\wv\|_2 \leq   \epsilon, \ \ \forall t.
	\label{eq.rea1}
\end{equation}
Observe that this {\em realisability} condition is quite restrictive.
However, it is prooven below that the algorithm also works well (`robust')
when deviating  from this assumption, allowing for the presence of faulty transactions during operation. 
First, a technical result is given. 
This result - as illustrated in Fig. (\ref{fig.bound})  - is paramount.
\begin{Lemma}[AC-lemma]
	Let $a,c>0$ be positive terms.
	Assume $x\geq 0$.
	If for all $y\in\R$ holds that  
	\begin{equation}
		x + y \leq c\sqrt{a + 2y},
		\label{eq.xy.1}
	\end{equation}
	and 
	\begin{equation}
		y\geq -\frac{a}{2}.
		\label{eq.xy.2}
	\end{equation}
	Then
	\begin{equation}
		|y| \leq \max\left( \frac{a}{2}, c \sqrt{a + c^2} +c^2\right),
		\label{eq.xy.3}
	\end{equation}
	and 
	\begin{equation}
		x \leq 
		\max\left( 
			c\sqrt{2a} + \frac{a}{2}, \ 
			c\sqrt{a + 2(c\sqrt{a  + c^2} +c^2)} +  c \sqrt{a + c^2}  +c^2\right).
		\label{eq.xy.4}
	\end{equation}
	In case $a = O(c^2)$, one has that $x = O(c^2)$.
\end{Lemma}
\begin{proof}
	First, using that $x\geq 0$, gives 
	\begin{equation}
		y \leq c\sqrt{a + 2y}.
		\label{eq.lmmxy.1}
	\end{equation}
	By inspection of either term of eq. (\ref{eq.lmmxy.1}) in terms of $y$ (see Figure (\ref{fig.bound}) 
	for an example of choices for $a,c$), one finds that one needs solving for the positive root of $y$ 
	in the equation resulting from replacing 
	$'\leq'$ by $'='$,
	or 
	\begin{equation}
		y = c\sqrt{a+2y} 
		\Leftrightarrow 		
		y^2 = c^2a+2c^2y \\
		\Leftrightarrow 		
		y^2 - 2c^2y  - c^2a=0, \\
		\label{eq.lmmxy.}
	\end{equation}
	and thus 
	\begin{equation}	
		y=c \sqrt{a + c^2} +c^2.
		\label{eq.lmmxy.}
	\end{equation}
	From this, it follows that 
	\begin{equation}
		y \leq c \sqrt{a + c^2} +c^2,
		\label{eq.lmmxy.}
	\end{equation}
	for whatever choice of $x$ in eq. (\ref{eq.xy.1}).
	Using the lowerbound of eq. (\ref{eq.xy.2}) as well, gives 
	\begin{equation}
		|y| \leq \max\left( \frac{a}{2}, c \sqrt{a + c^2} +c^2\right),
		\label{eq.lmmxy.}
	\end{equation}
	and thus
	\begin{equation}
		x \leq c\sqrt{a + 2|y|}+ |y| \\
		\Rightarrow
		x \leq 
			\begin{cases}
			c\sqrt{2a} + \frac{a}{2}& \\
			c\sqrt{a + 2(c\sqrt{a  + c^2} +c^2)} &\\ 
			\hspace{+10mm}+  c \sqrt{a + c^2}  +c^2. & \\
			\end{cases}
		\label{eq.lmmxy.}
	\end{equation}
	This gives the result.
	In case $a=O(c^2)$, one has in brief that $x=O(c^2)$.
\end{proof}

\begin{figure}[htbp] %  figure placement: here, top, bottom, or page
   \centering
   \includegraphics[width=5in]{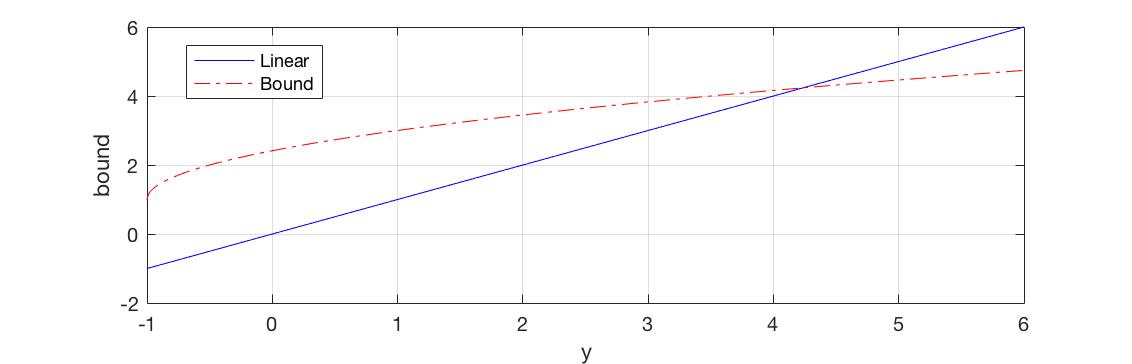} 
   \caption{Illustration of the AC-Lemma: when a quantity $y$ (blue solid line) is bounded 
   by a constant plus the square-root of itself (red dashed line), the quantity must be bounded (here $y\leq 4.2$).}
   \label{fig.bound}
\end{figure}

\begin{Lemma}
	Assume that 
	a fixed $\bar\wv$ exists for given $\epsilon$ 
	such that eq. (\ref{eq.rea1}) holds for all transactions.
	Let $\{\gamma_t>0\}_t$ be any sequence, then
	\begin{equation}
		\sum_{s=1}^t d_s \gamma_s  \vv_s^T\wv_{s-1} \geq 
		-\frac{1}{2}\sum_{s=1}^t d_s^2 \gamma_s^2. 
		\label{eq.lower}
	\end{equation}
\end{Lemma}
This means that the terms $\{\sum_{s=1}^t d_s \gamma_s  \vv_s^T\wv_{s-1} \}_t$
cannot become too negative.

\begin{proof}
	This follows by studying the positive terms $\{\|\wv_t\|_2\}_t$.
	The algorithm can be summarised as the recursion
	\begin{equation}
		\wv_t = \wv_{t-1} + d_t\gamma_t \vv_t = 
		\begin{cases}
		 	\wv_{t-1} +\gamma_t \vv_t & \ \mbox{if \ } d_t=1 \\
			\wv_{t-1} & \ \mbox{if \ } d_t=0,
		\end{cases}
		\label{eq.proof0.}
	\end{equation}
	which is initialised as $\wv_0=0_n\in\R^n$.
	The algorithm implements the recursion for any $t>0$ that 
	\begin{equation}
		\wv_t = \wv_{t-1} + d_t \gamma_t \vv_t.
		\label{eq.proof0.}
	\end{equation}	
	From this it follows that for any $t$ one has 
	\begin{multline}
		 0\leq \|\wv_t\|_2^2  
		 = 
		 \|\wv_{t-1} + d_t\gamma_t\vv_t\|_2^2 \\
		=
		 \|\wv_{t-1}\|_2^2 + d_t^2 \gamma_t^2\|\vv_t\|_2^2 
		 + 2d_t \gamma_t \wv_{t-1}^T\vv_t\\
		=
		\sum_{s=1}^t d_s^2 \gamma_s^2 + 
		2\sum_{s=1}^t d_s\gamma_s ( \wv_{s-1}^T\vv_s).
		\label{eq.proof0.}
	\end{multline}	
	Reshuffling terms gives the result.
\end{proof}
Now an upperbound is given, meaning that the terms 
$\{\sum_{s=1}^t d_s \gamma_s  \vv_s^T\wv_{s-1} \}_t$
cannot become too large.

\begin{Lemma}
	Assume that 
	a fixed $\bar\wv$ exists for given $\epsilon$ 
	such that eq. (\ref{eq.rea1}) holds for all transactions.
	Let $\{\gamma_t>0\}_t$ be any sequence, then
	\begin{equation}
		\sum_{s=1}^t d_s \gamma_s  \vv_s^T\wv_{s-1}
		\leq  
		\|\bar\wv\|_2^2 + \|\bar\wv\|_2\sqrt{\sum_{s=1}^t d_s\gamma_s^2 + \|\bar\wv\|_2^2}.
		\label{eq.proof0.11}
	\end{equation}
\end{Lemma}
\begin{proof}
	This follows from studying the terms $\{\bar\wv^T\wv_t\}_t$.
	First
	\begin{equation}
		d_t(\vv_t^T\bar\wv - \vv_t^T\yv_t) 
		\geq 
		-d_t^2\|\bar\wv-\yv_t\|_2 
		\geq 
		-d_t^2\epsilon,
		\label{eq.proof0.}
	\end{equation}	
	by construction, and hence that
	\begin{equation}
		\wv_t^T\bar\wv \geq 
			\sum_{s=1}^t d_s \gamma_s  \vv_s^T\yv_s  - 
			\sum_{s=1}^t d^2_s \gamma_s \epsilon.
		\label{eq.proof0.}
	\end{equation}	
	Since in cases $t$ where a mistake was reported, or $d_t=1$, one has 
	\begin{equation}
		 \vv_t^T\yv_t  - \vv_t^T\wv_{t-1} = \|\yv_t - \wv_{t-1}\|_2 \geq \epsilon, 
		\label{eq.proof0.}
	\end{equation}	
	or
	\begin{equation}
		 \vv_t^T\yv_t  -  \epsilon \geq \vv_t^T\wv_{t-1}, 
		\label{eq.proof0.}
	\end{equation}	
	one also has 
	\begin{equation}
		\wv_t^T\bar\wv \geq 
			\sum_{s=1}^t d_s \gamma_s  (\vv_s^T\wv_{s-1}). 
		\label{eq.proof0.}
	\end{equation}	
	Conversely, one has using Cauchy-Schwarz' inequality that 
	\begin{equation}
		\left|\wv_t^T\bar\wv\right| \leq \|\bar\wv\|_2  \|\wv_t\|_2.
		\label{eq.proof0.}
	\end{equation}	
	Since 
	\begin{multline}
		 \|\wv_t\|_2^2  
		 = 
		 \|\wv_{t-1} + d_t\gamma_t\vv_t\|_2^2 \\
		=
		 \|\wv_{t-1}\|_2^2 + d_t^2\gamma_t^2\|\vv_t\|_2^2 
		 + 2d_t \gamma_t \wv_{t-1}^T\vv_t\\
		=
		\sum_{s=1}^t d_s^2 \gamma_s^2 + 
		2\sum_{s=1}^t d_s\gamma_s ( \wv_{s-1}^T\vv_s).
		\label{eq.proof0.}
	\end{multline}	
	Putting together those inequalities gives 
	\begin{equation}
		\sum_{s=1}^t d_s \gamma_s  \vv_s^T\wv_{s-1}
		\leq 
		\wv_t^T\bar\wv \\ 
		\leq 
		\|\bar\wv\|_2 
		\sqrt{\sum_{s=1}^t d_s^2\gamma_s^2 + 
			2\sum_{s=1}^t d_s\gamma_s (\vv_s^T\wv_{s-1})},
		\label{eq.proof0.}
	\end{equation}	
	then using the AC-Lemma with $x=0$,
	$a=\sum_{s=1}^t d_s^2\gamma_s^2$ and 
	$c=\|\bar\wv\|_2$ gives the result.
\end{proof}

 Now, we consider a specific function of $\gamma_s>0$:
 \begin{Lemma}
	Let $\tau>0$  be a strictly positive, finite constant, and 
	\begin{equation}
		\gamma_t = \frac{1}{m_t^{1/2+\tau}} \ \ \forall t,
		\label{eq.gamb}
	\end{equation}
	then
	\begin{equation}
		\sum_{s=1}^t d_s\gamma_s^2 \leq  \zeta(1+2\tau) < \infty,
		\label{eq.gam2}
	\end{equation}
	where $\zeta(\cdot)$ is the Riemann zeta function, and one has for all $t>0$ that
	\begin{equation}
		\sum_{s=1}^t d_s\gamma_s \geq  cm_t^{1/2-\tau}.
		\label{eq.gam3}
	\end{equation}
	with $c>0$ a universal constant.
\end{Lemma}
%The previous results sandwich the terms
%$\{\sum_{s=1}^t d_s \gamma_s  \vv_s^T\wv_{s-1} \}_t$,
%can still oscillate, meaning that still $m_t\rightarrow\infty$.
%We observed that this can indeed happen under specific practical scenarios.
%
In order to bound $m_t$, we need to make stronger assumptions.
For example, that there exists a $\mu>0$ such that
\begin{equation}
	\| \yv_t - \bar\wv\|_2 \leq  \epsilon - \mu, \ \ \forall t.
	\label{eq.rea2}
\end{equation}
That is, the detection rule achieves perfect performance with a non-zero margin $\mu$.
Then, we obtain actually the following stronger result:
\begin{Theorem}
	Assume that 
	a fixed $\bar\wv$ exists for given $\epsilon$ and $\mu>0$
	such that eq. (\ref{eq.rea2}) holds for all transactions $\yv_1, \dots,\yv_t, \dots$ presented.
	Let $\{\gamma_t>0\}_t$ be set as
	 \begin{equation}
		\gamma_t = \frac{1}{m_t^{1/2+\tau}}, \ \ \forall t>0,
		\label{eq.gam}
	\end{equation}
	then
	\begin{equation}
		m_t
		= 
		O\left(\left( \frac{\|\bar\wv\|_2^2}{\mu} \right)^{\frac{2}{1-2\tau}} \right) < \infty, \ \forall t>0.
	\label{eq.bound}
	\end{equation}
\end{Theorem}
Figure (\ref{fault2d}) gives an example run of the algorithm in $n=2$ dimensions.

  \begin{figure}[htbp]
	\begin{center}
		\includegraphics[width=0.6\textwidth]{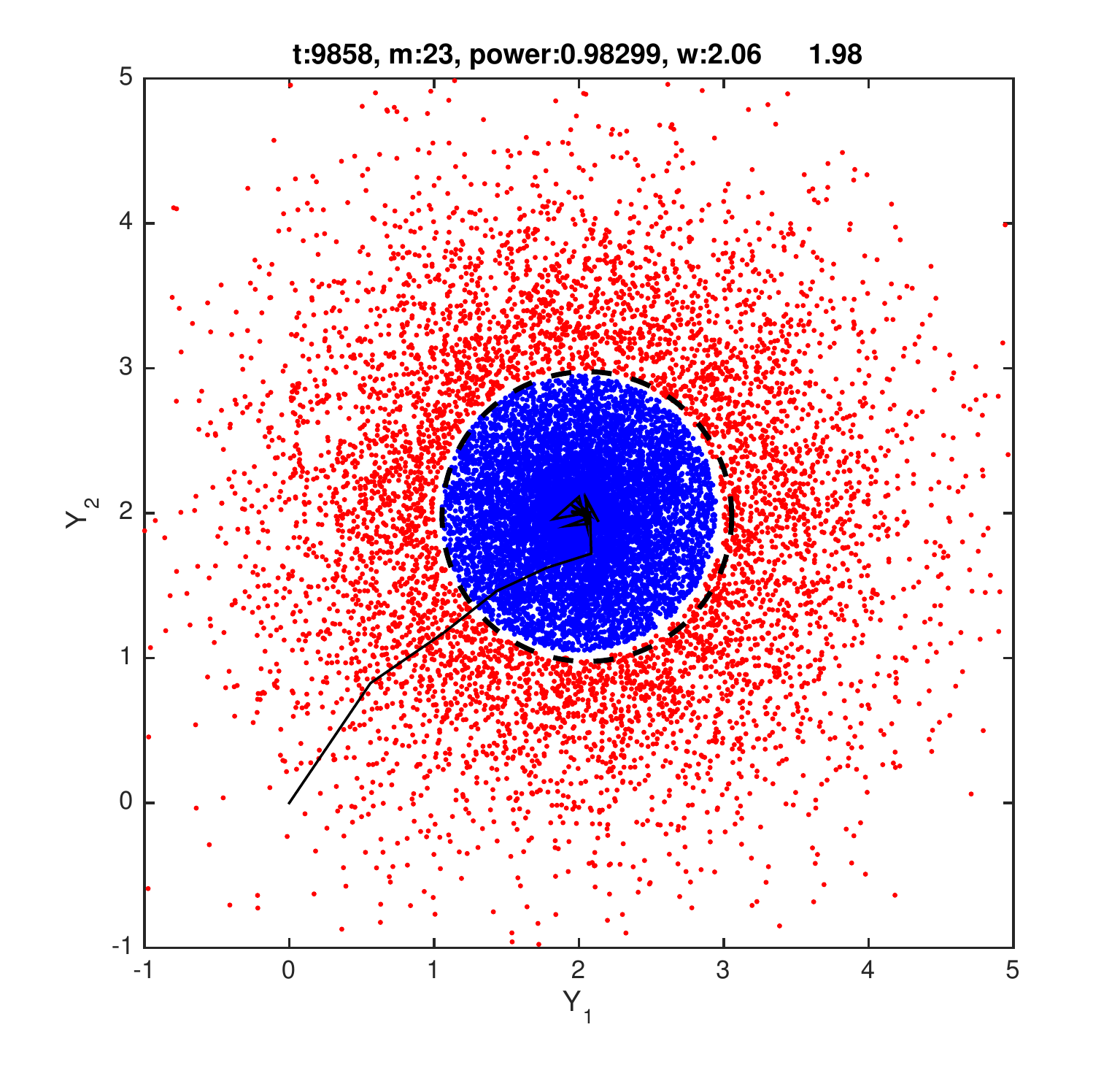}
	\end{center}
	\caption{\em 
		Example in 2D. Only $m=23$ of the $t=10.000$ `normal' 
		cases (blue dots) were erroneously decided to be faulty by the 
		FADO detection algorithm. 
		The final rule (dashed circle) detects	98.299\% of the outliers (red dots) as atypical.
		The solid path depicts the different centres the FADO algorithm has taken over time ($t=1, \dots, 10.000$).
		That is, the centre was first taken to be $(0,0)$ and evolved after $23$ mistakes
		 into the value $\wv_t=(2.06,1.98)$, close to $\bar\wv=(2,2)$.
	}
	\label{fault2d}
\end{figure}

This theorem implies  that only a limited number of false detections are made however many 
`normal' cases are presented to the FADO($\epsilon$) algorithm.
The $O(\cdot)$ suppresses the explicit statement of the (mild) constants in the bound.
These constants depend on the term `$\sum_{s=1}^t d_s\gamma_s^2<\infty$' as in Lemma 3, 
and hence directly on the choice of $\{\gamma_t\}_t$.

The argument implies that the larger one chooses $\tau$, the worse the behaviour.
However the constants in the bound become smaller for larger $\tau>0$, 
hence improving the bound in that sense. 
So, one needs to tune  $0<\tau<\frac{1}{2}$ to trade-off both effects 
optimally in the case at hand.
In practice this does not affect results too much, and the key property is 
that $\sum_t d_t\gamma_t^2<\infty$ while $\sum_t d_t\gamma_t\rightarrow\infty$.
This makes this result reminiscent to the beautiful result of \cite{robbins1951stochastic}.

Note that the bound of Theorem 1 is not dependent on the 
{\em size} of the domain of $\yv_t$. 
That is, the performance of the algorithm is not different if the data were to be rescaled.
This is a consequence of the fact that the update $\vv_t$ has unit norm,
(i.e., this is very similar to the robust {\em median} estimate),
rather than to allow for individual influences which are proportional to the norm of the mistake
(an approach which is more similar to the {\em mean} or least-squares estimate instead.)
This is useful, as in this setting 
one is interested in detecting outliers, and putting restrictions on the (size of)
possible data is not realistic.

\begin{proof} (of Theorem 1) 
	The proof goes along the same lines of the proof of Lemma 2,
	however strengthening the first inequality as 
	\begin{equation}
		d_t(\vv_t^T\bar\wv - \vv_t^T\yv_t) 
		\geq 
		-d_t^2\|\bar\wv-\yv_t\|_2 
		\geq 
		-d_t^2\epsilon + d_t^2\mu,
		\label{eq.proof4.}
	\end{equation}
	Working through the following inequalities then results in
	\begin{equation}
		\sum_{s=1}^t d_s \gamma_s  \vv_s^T\wv_{s-1}
		+ \sum_{s=1}^t d_s \gamma_s \mu \\
		\leq 
		\|\bar\wv\|_2 
		\sqrt{\sum_{s=1}^t d_s^2\gamma_s^2 + 
			2\sum_{s=1}^t d_s\gamma_s (\vv_s^T\wv_{s-1})},
		\label{eq.proof0.}
	\end{equation}	
	Application of the AC-Lemma for 
	$x=\sum_{s=1}^t d_s \gamma_s \mu$,
	$y=\sum_{s=1}^t d_s \gamma_s  \vv_s^T\wv_{s-1}$
	and 
	$a=\sum_{s=1}^t d_s^2\gamma_s^2$,
	$c=\|\bar\wv\|_2$,
	gives then
	\begin{equation}
		\sum_{s=1}^t d_s \gamma_s \mu 
		\leq 
		O\left(\|\bar\wv\|_2^2\right).
		\label{eq.proof0.}
	\end{equation}	
	and by the specific choice of $\gamma_t$ that
	\begin{equation}
		m_t
		= 
		O\left(\left( \frac{\|\bar\wv\|_2^2}{\mu} \right)^{\frac{2}{1-2\tau}} \right),
		\label{eq.proof0.}
	\end{equation}	
	as desired.
\end{proof}

  \begin{figure}[htbp]
	\begin{center}
		\includegraphics[width=0.7\textwidth]{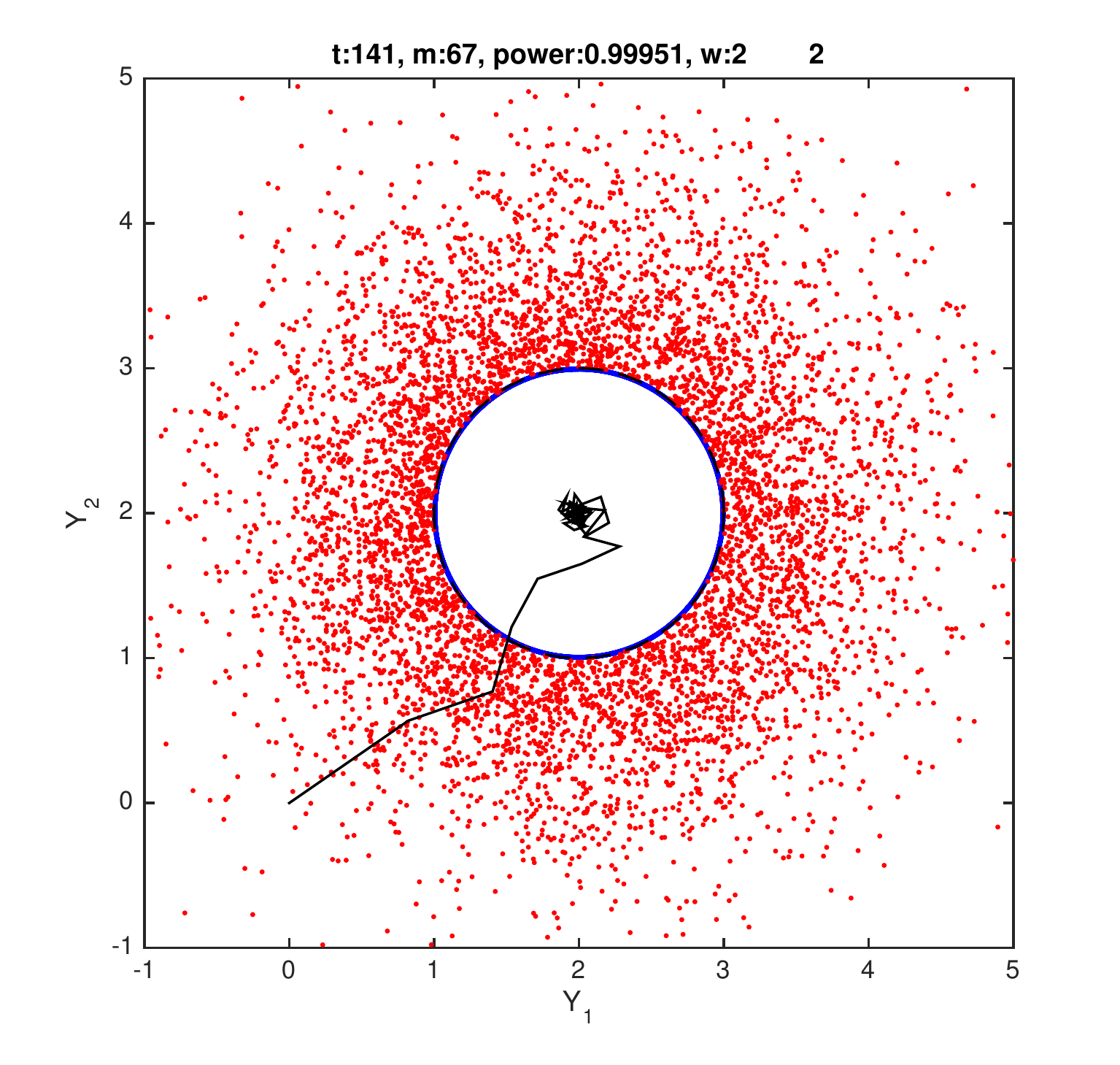} \\
		\includegraphics[width=0.7\textwidth]{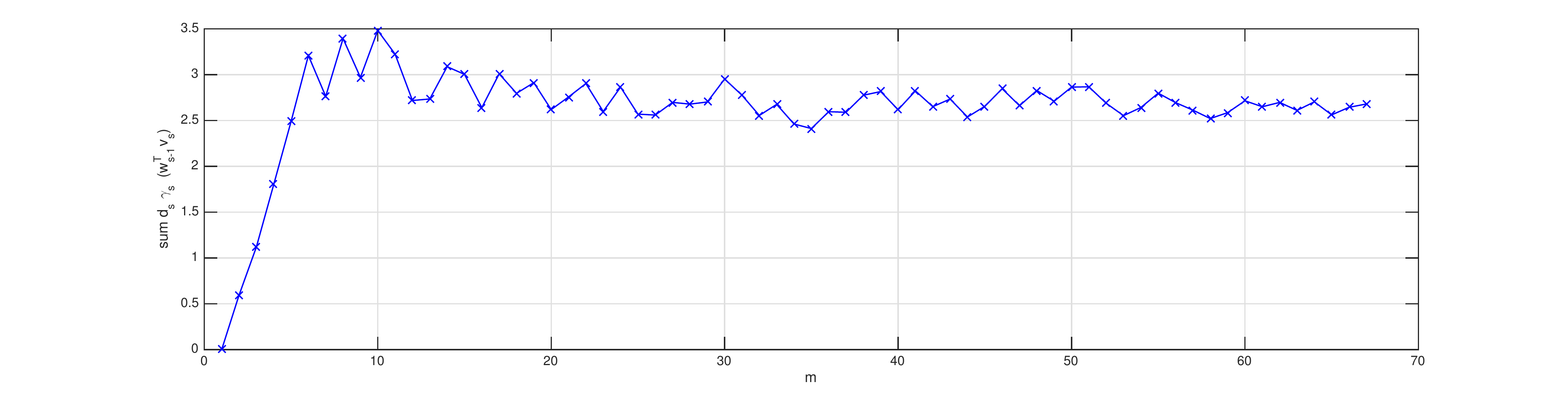} \\
		\includegraphics[width=0.7\textwidth]{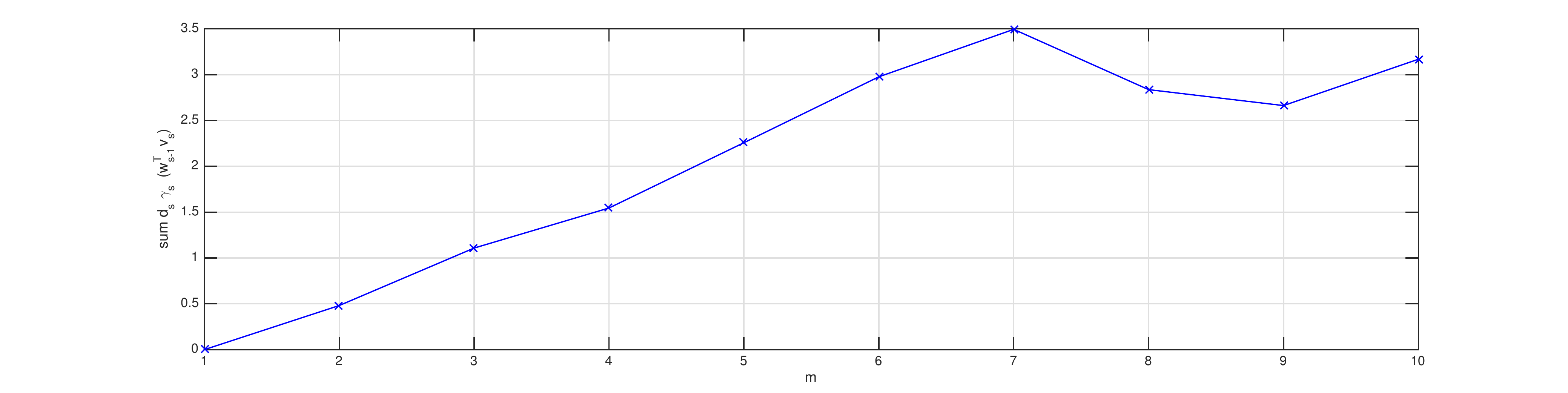}

	\end{center}
	\caption{\em 
		Example in 2D where the $\{\yv_t\}$ are on a circle (blue dots)
		quite close to the decision boundary ($\mu= 0.001$).
		(subplot a): The red dots denote the abnormal transactions not seen, but to be detected. 
		This results in $m=67$ mistakes.
		Subplot (b) reports the evolution of the term $\sum_s d_s\gamma_s (\wv_{s-1}^T\vv_s)$,
		indicating that it remains bounded.
		After these $m=67$ corrections, a decision function is obtained with almost perfect 
		detection capabilities (i.e. the power for detecting the red dots in 
		subplot a as faulty, is $99.951\%$).		
		Whenever this margin $\mu$ is taken larger, $m$ decreases as e.g. seen in subplot (c), 
		(same design, but here $\mu=0.1$, and only $m=10$ false detections are made).
	}
	\label{fault2db}
\end{figure}

\subsection{Power of FADO($\epsilon$)}
%%%%%%%%%%%%%%%%%%%%

Now, we proof that the algorithm has sufficient power,
i.e. it detects items $\yv$ which violate eq. (\ref{eq.rea1}) sufficiently much.
Hereto we consider the following scenario. 
Let $T>0$ transactions be given satisfying the assumption (\ref{eq.rea1}) for $(\bar\wv,\epsilon)$. 
Then we study the performance of the FADO($\epsilon$) 
algorithm on the $(T+1)$th transaction $\yv_{T+1}$ 
which is {\em significantly} atypical: i.e. there exist a $\delta\geq 0$ so that 
\begin{equation}
	\| \bar\wv - \yv_{T+1}\|_2 = \epsilon + \delta.
	\label{eq.sa}
\end{equation}
Then we ask the question: {\em How large needs this $\delta\geq0$ to be so that 
FADO($\epsilon$) 
detects $\yv_{T+1}$ as a mistake?}, i.e. 
\begin{equation}
	\| \wv_T - \yv_{T+1}\|_2 \geq \epsilon.
	\label{eq.sa2}
\end{equation}
The smaller this $\delta\geq 0$ is, the more {\em powerful} the algorithm is.
Observe that a trivial solution $(\wv_T,\epsilon_T)$ scoring  {\em all possible}
transactions as {\em faulty}, has perfect power ($\delta=0$) in this sense.
Such solution will also have an unbounded number of false positives 
- as characterised by Theorem 1 - and is hence not very useful. 
The optimal trade-off between false positives and false negatives 
(or number of mistakes versus power, under assumption eq. (\ref{eq.rea1})) 
is dictated by the application at hand.

\begin{Theorem}
	Given a fixed $T>0$.
	Assume the existence of $\bar\wv$ for a given $\epsilon$ 
	such that eq. (\ref{eq.rea1}) holds for all transactions $\yv_1, \dots, \yv_T$.
	Then the FADO($\epsilon$) algorithm will detect the transaction $\yv_{T+1}$ as faulty 
	whenever
	\begin{equation}
		\| \bar\wv - \yv_{T+1}\|_2 = \epsilon + \delta,
		\label{eq.sa3}
	\end{equation}
	and 
	\begin{equation}
		\delta \geq O\left(\frac{\|\bar\wv\|_2^2}{m_T^{\frac{1-2\tau}{4}} }\right).
		\label{eq.sa4}
	\end{equation}
	That is, the algorithm is $\delta$-powerful.	
\end{Theorem}

\begin{proof}
	The proof goes by contradiction, i.e. assume $\yv_{T+1}$
	satisfies eq. (\ref{eq.sa3}) where $\delta$ satisfies eq. (\ref{eq.sa4}), and assume that
	\begin{equation}
		\| \wv_T - \yv_{T+1}\|_2 < \epsilon.
		\label{eq.proof3.1}
	\end{equation}
	Now we proof that this scenario is impossible.
	Define $\bar\vv_{T+1}$ as 
	the (unknown) sub-gradient according to $\bar\wv$, or 
	\begin{equation}
		\bar\vv_{T+1}^T\left(\yv_{T+1} - \bar\wv\right) 
		=
		\left\|\yv_{T+1} - \bar\wv\right\|_2,
		\label{eq.proof3.2}
	\end{equation}
	or 
	\begin{equation}
		\bar\vv_{T+1} 
		=
		\frac{\yv_{T+1} - \bar\wv}
			{\left\|\yv_{T+1} - \bar\wv\right\|_2},
		\label{eq.proof3.3}
	\end{equation}
	if $\|\bar\wv-\yv_{T+1}\|_2>0$. 
	Note  that the case $\|\bar\wv-\yv_{T+1}\|_2=0$ is impossible since it is assumed that $\|\bar\wv-\yv_{T+1}\|_2>\epsilon$.
	Note that in any case $\|\bar\vv_{T+1}\|_2=1$.
	Now we proceed as before, however studying 
	$(\wv_T - c\bar\vv_{T+1})^T\bar\wv$ 
	where $c>0$ is set lateron.
	Since 
	\begin{equation}	
		\bar\wv^T\wv_T 
		= 
		\sum_{s=1}^T d_s \gamma_s \vv_s^T\bar\wv 
		\geq 		
		\sum_{s=1}^T d_s \gamma_s (\vv_s^T\yv_s - \epsilon),
		\label{eq.proof3.4}
	\end{equation}
	and by definition of $\bar\vv_{T+1}$, it follows that 
	\begin{equation}		
		\bar\wv^T(\wv_T + c \bar\vv_{T+1}) 
		\geq 		
		\sum_{s=1}^{T} d_s \gamma_s (\vv_s^T\yv_s - \epsilon) - c \bar\wv^T\bar\vv_{T+1}.
		\label{eq.proof3.5}
	\end{equation}
	By definition of $d_s=1$ and thus $\vv_s^T(\yv_s-\wv_{s-1})\geq \epsilon$,
	and the assumption that no mistake was made at instance $T+1$ despite 
	the fact that 
	\begin{equation}
		\bar\vv_{T+1}^T(\yv_{T+1} - \bar\wv) 
		=
		\|\bar\wv - \yv_{T+1}\|_2
		= 
		\epsilon+\delta,
		\label{eq.proof3.6}
	\end{equation}
	then one has 
	\begin{equation}		
		\bar\wv^T(\wv_T + c \bar\vv_{T+1}) \\
		\geq 
		\sum_{s=1}^{T} d_s \gamma_s (\vv_s^T\wv_{s-1}) 
		- c\bar\vv_{T+1}^T\yv_{T+1}  
		+ c\epsilon + c\delta.
		\label{eq.proof3.7}
	\end{equation}
	Since by assumption (contradiction), no mistake is made at $T+1$, one has 
	\begin{equation}		
		\bar\vv_{T+1}^T (\wv_T - \yv_{T+1}) 
		\geq 
		-\|\yv_{T+1} - \wv_T\|_2^2 
		\geq 
		-\epsilon,
		\label{eq.proof3.8}
	\end{equation}
	or	
	\begin{equation}		
		- \bar\vv_{T+1}^T\yv_{T+1} + \epsilon 
		\geq 
		-\vv_{T+1}^T\wv_T.
		\label{eq.proof3.9}
	\end{equation}
	Hence
	\begin{equation}		
		\bar\wv^T(\wv_T - \gamma_{T+1}\bar\vv_{T+1}) \\
		\geq 
		\sum_{s=1}^{T} d_s \gamma_s (\vv_s^T\wv_{s-1}) 
		- c\bar\vv_{T+1}^T\wv_T + c \delta.
		\label{eq.proof3.10}
	\end{equation}	
	Conversely, we have that 
	\begin{multline}		
		\|\wv_T -c\bar\vv_{T+1}\|_2^2
		= \|\wv_T\|_2^2 +  c^2 \|\bar\vv_{T+1}\|_2^2 
		- 2c\wv_T^T\bar\vv_{T+1} \\
		=
		\sum_{s=1}^{T} d_s\gamma_s^2 
		+ 
		2\sum_{s=1}^{T} d_s \gamma_s (\vv_s^T\wv_{s-1}) 
		+ c^2
		- 2c \wv_T^T\bar\vv_{T+1}.
		\label{eq.proof3.11}
	\end{multline}
	Combining inequalities gives 
	\begin{multline}		
		\left(\sum_{s=1}^{T} d_s \gamma_s (\vv_s^T\wv_{s-1}) 
		- c \bar\vv_{T+1}^T\wv_{T}\right)
		+ \delta c \\
		\leq 
		\|\bar\wv\|_2 
		\sqrt{
		\left(\sum_{s=1}^{T} d_s\gamma_s^2 +c^2\right)
		+ 
		2\left(\sum_{s=1}^{T} d_s \gamma_s (\vv_s^T\wv_{s-1}) 
		- c \bar\vv_{T+1}^T\wv_{T}\right)}.
		\label{eq.proof3.12}
	\end{multline}
	Then application of the AC-Lemma
	for $x=c\delta$, 
	$y=\left(\sum_{s=1}^{T} d_s \gamma_s (\vv_s^T\wv_{s-1}) - c\bar\vv_{T+1}^T\wv_{T}\right)$, 
	$c=\|\bar\wv\|_2$ and 
	$a=\left(\sum_{s=1}^{T} d_s\gamma_s^2 +c^2\right)$,
	gives 
	\begin{equation}	
		c \delta \leq 	O(\|\bar\wv\|_2^2),
		\label{eq.proof3.13}
	\end{equation}
	whenever $a=\left(\sum_{s=1}^{T} d_s\gamma_s^2 +c^2\right) = O\left(\|\bar\wv\|_2^2\right) <\infty$.
	Now we turn to the question how to choose $c$. The larger $c$ is, the better the power (the smaller $\delta$ should be).
	However, this constant cannot be too large as then the condition $a=O(c^2)$ of the AC-Lemma
	would be violated, hence the maximal value of $c$ is such that 
	\begin{equation}	
		\left(\sum_{s=1}^{T} d_s\gamma_s^2 +c^2\right) = O\left(\|\bar\wv\|_2^2\right) \leq \infty.
		\label{eq.proof3.14}
	\end{equation}
	For example, if $c$ would depend on $T$, this condition is violated.
	Since the Theorem 1 states that $m_T^{\frac{1-2\tau}{2}}=O(\|\bar\wv\|_2^2)$ for a fixed $\mu>0$, 
	one may choose $c=O(m_T^{\frac{1-2\tau}{4}})$ such that $c^2=O(\|\bar\wv\|_2^2)$.
	Since this contradicts $\delta$ as in eq. (\ref{eq.sa4}), the result follows.
\end{proof}

\subsection{Robustness and the Agnostic Case}
%%%%%%%%%%%%%%%%%%%%%

The previous results were obtained by relying on the realisability 
assumption of eq. (\ref{eq.rea1}) or even eq. (\ref{eq.rea2}).
Results however degrade not too fast when deviating from this assumption.
First, we present numerical results indicating how fast results degrade in practical situations 
where the realisability assumption only holds vaguely. 

\begin{figure}[htbp]
	\begin{center}
		\includegraphics[width=0.75\textwidth]{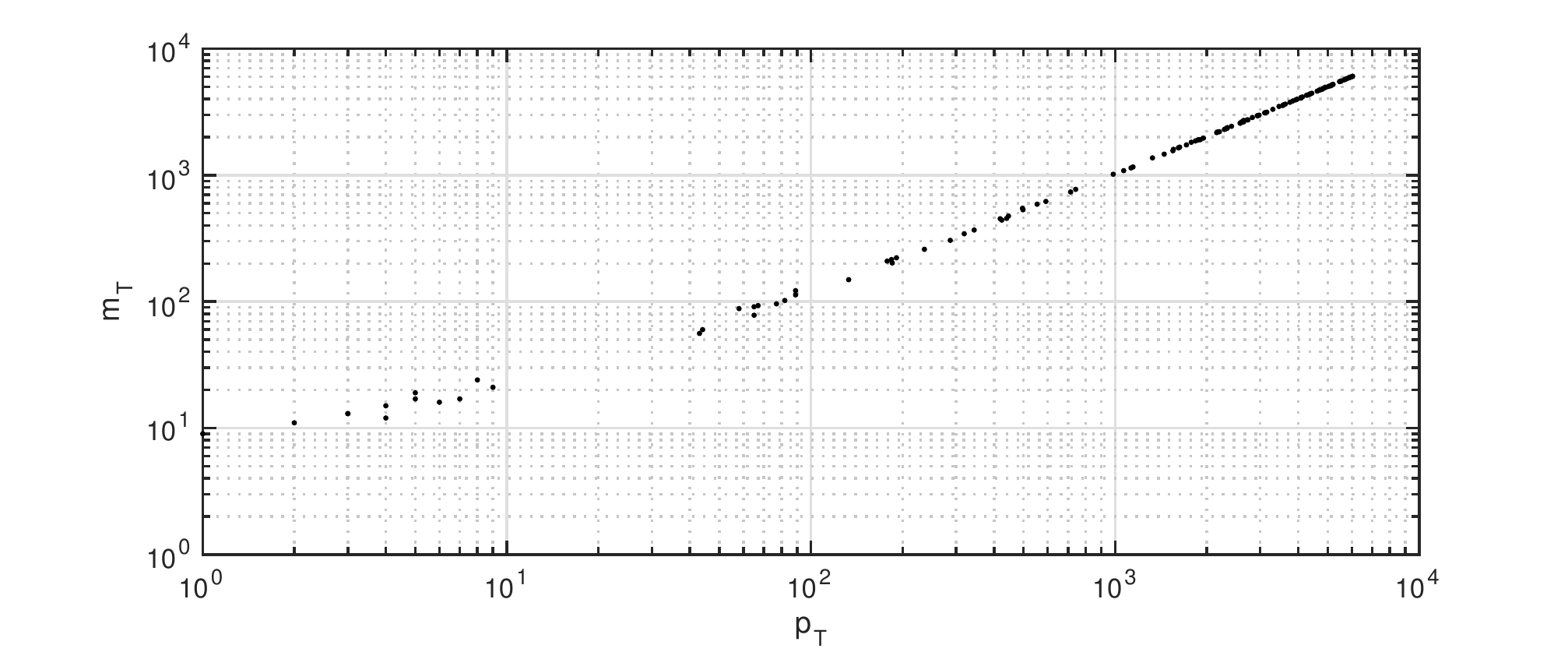} \\
		\includegraphics[width=0.75\textwidth]{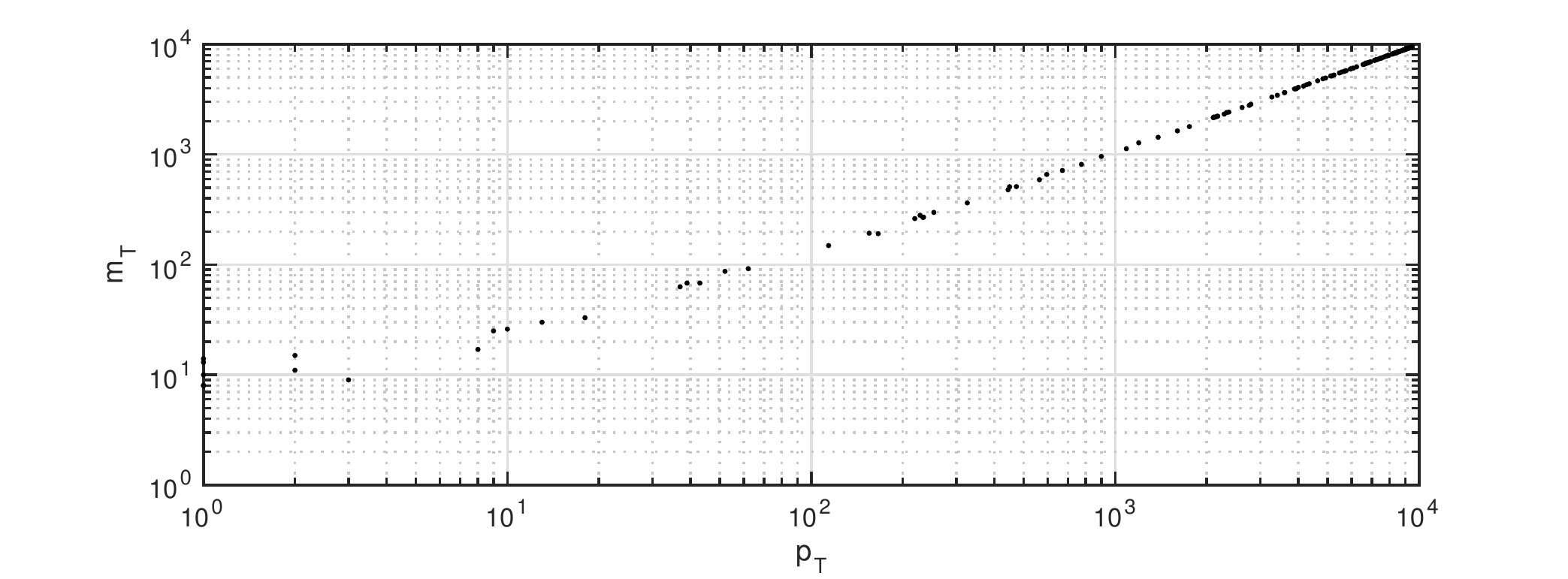}
	\end{center}
	\caption{\em 
		Number of false negatives $m_T$ in case of $p_T$ 
		faulty samples/fraudulent transactions.
		Subplot (a) for $T=10000$ and $n=2$,
		Subplot (b) for $T=10000$ and $n=10$.		
	}
	\label{dev}
\end{figure}

Figure (\ref{dev}.a) and (\ref{dev}.b) present results of a numerical study 
based on $T=10000$ samples $\{\yv_t\}_{t=1}^T$
where the number $m_T$ is plotted in function of the 
number of cases violating the realisability assumption of eq. (\ref{eq.rea2}). 
This number $p_T$ is defined as 
\begin{equation}
	p_T = \left|\{1\leq t \leq T: \|\yv_t- \bar\wv\|_2\geq \epsilon\}\right|,
	\label{eq.p}
\end{equation}
where $|\cdot|$ denotes the number of elements in a set.
Figure (\ref{dev}.a) presents this graph for the case where 
$\{\yv_t\}_t\subset\R^n, \wv=(2,2, 0, \dots, 0)^T\in\R^n$ and 
$n=2$, with a design as given used in (\ref{fault2db}).
Figure (\ref{dev}.b) presents results in case of $n=10$ where now $(2,2,0, \dots, 0)^T\in\R^{10}$.
Results indicate that performance 
(characterised as $m_T$, the number of detections in the set $\{\yv_t\}$)
increases directly proportional to the number of cases violating the realisability assumption.
This means that useful results are obtained even for situations 
where most (!) of the cases in $\{\yv_t\}_t$ are faulty themselves. 

From this plot, we see that $m_T$ is directly proportional to $p_T$ in practical scenarios.
The theoretical bound based on the above arguments gives a slightly worse dependence on $p_T$
as follows. This bound is in terms of the total {\em size} of the `faulty items' in $P_T$,
not in terms of their number.

Now we operate under the following assumption, namely that for given $\epsilon,\mu>0$
there exists a vector $\bar\wv\in\R^n$ where for a set $P_T\in\{1, \dots, T\}$ of size $p_t=|P_T|$
\begin{equation}
	 \|\yv_t- \bar\wv\|_2\leq \epsilon-\mu, \forall t\not\in P_T.
	\label{eq.nonrea}
\end{equation}
That is, the decision rule works for a fraction of 
\begin{equation}
	r_{T}=\left(1-\frac{p_T}{T}\right).
	\label{eq.fraction}
\end{equation}
Moreover, let the {\em size} of the faults occurring before item $t$ be denoted as $\sigma_t$
defined as 
\begin{equation}
	\sigma_t = \sum_{s=1}^t \left(\|\yv_s-\bar\wv\|_2 - (\epsilon-\mu)\right)_+^2,
	\label{eq.size}
\end{equation}
with $(\cdot)_+$ denoting the positive part of its argument.
Note the `square' in the definition.
Then
\begin{Theorem}
	Fix $\epsilon,\mu>0$ and $p_T$ such that there exists a $\bar\wv$ and $P_T$  
	of size $p_T$ for which eq. (\ref{eq.nonrea}) holds.
	Assume bounded transactions where $\max_t\|\yv_t\|_2<\infty$.
	Let $\tau>0$ be small, and let $\{\gamma_t>0\}_t$ be choosen as 
	 \begin{equation}
		\gamma_t = \frac{1}{m_t^{1/2+\tau}}, \ \ \forall t>0,
		\label{eq.a.gam}
	\end{equation}
	and assume that 
	 \begin{equation}
		\sigma_T = O\left(\|\bar\wv\|_2^8\right) < \infty,
		\label{eq.a.rho}
	\end{equation}
	then FADO($\epsilon$) can make at most 
	\begin{equation}
		m_t^{1-2\tau} =  O\left(\frac{ \|\bar\wv\|_2^4}{\mu^2} \right) \ \forall t>0,
		\label{eq.a.bound}
	\end{equation}
	mistakes.
\end{Theorem}
\begin{proof}
	The key element is in the first step of the proof, namely in deriving the lower-bound.
	First, for cases $t$ not in $P_T$, one has
	\begin{equation}
		\vv_t^T(\bar\wv - \yv_t) \geq - \|\yv_t - \bar\wv\|_2 \geq -\epsilon + \mu.
		\label{eq.proof5.1}
	\end{equation}
	For cases $t$ within $P_T$ (`faulty cases') 
	\begin{equation}
		\vv_t^T(\bar\wv - \yv_t) \geq - \|\yv_t - \bar\wv\|_2 \geq -\epsilon + \mu - \delta_t
		\label{eq.proof5.2}
	\end{equation}
	where $\delta_t\geq 0$ denotes the `size' of this mistake, or $\vv_t^T\bar\wv \geq \vv_t^T\yv_t - \epsilon + \mu - d_t\delta_t$.
	Considering those cases $t$ were a mistake was made ($d_t=1$), one has 
	\begin{equation}
		\vv_t^T(\yv_t - \wv_{t-1}) = \|\yv_t - \wv_{t-1}\|_2 \geq \epsilon,
		\label{eq.proof5.3}
	\end{equation}
	or $\vv_t^T\yv_t - \epsilon  \geq \vv_t^T\wv_{t-1}$.
	Hence
	\begin{equation}
		\bar\wv^T\wv_{t} 
		= 
		\sum_{s=1}^t d_s \gamma_s \vv_s^T\bar\wv \\
		\geq 
		\sum_{s=1}^t d_s \gamma_s \vv_s^T\wv_{s-1}
		+
		\sum_{s=1}^t d_s \gamma_s \mu
		- \sum_{s=1}^t d_s \gamma_s \delta_s.
		\label{eq.proof5.4}
	\end{equation}
	where we let $\delta_t=0$ in case instance $t$ is not `faulty' (not in $P_T$).
	Thus
	\begin{equation}
		\bar\wv^T\wv_{t} 
		\geq 
		\sum_{s=1}^t d_s \gamma_s \vv_s^T\wv_{s-1}
		+
		\sum_{s=1}^t d_s \gamma_s \mu
		- \sqrt{\zeta(1+2\tau)\sigma_t}.
		\label{eq.proof5.5}
	\end{equation}
	where we use Cauchy-Schwarz' inequality such that 
	\begin{equation}
		 \sum_{s=1}^t d_s \gamma_s \delta_s 
		 \leq \sqrt{ \sum_{s=1}^t d_s^2 \gamma_s^2} \sqrt{ \sum_{s=1}^t \delta_s^2}
		 \leq \sqrt{\zeta(1+2\tau) \sigma_t}.
		\label{eq.proof5.6}
	\end{equation} 
 	Since the upper-bound to $\bar\wv^T\wv_{t}$ is the same as 
	before, we have 
	\begin{equation}
		\sum_{s=1}^t d_s \gamma_s  \vv_s^T\wv_{s-1}
		+ \sum_{s=1}^t d_s \gamma_s \mu 
		- \sqrt{\zeta(1+2\tau) \sigma_t}
		\\
		\leq 
		\|\bar\wv\|_2 
		\sqrt{\sum_{s=1}^t d_s^2\gamma_s^2 + 
			2\sum_{s=1}^t d_s\gamma_s (\vv_s^T\wv_{s-1})}.
		\label{eq.proof5.7}
	\end{equation}	
	Application of the AC-Lemma assuming that $\sqrt{\sigma_t}=O(\|\bar\wv\|_2^4)<\infty$, gives then the result.
\end{proof}

Remark that the choice of $\epsilon$ influences the bound only indirectly.
If $\epsilon>0$ was set unrealistically small, many cases in $\{\yv_t\}$ would violate 
eq. (\ref{eq.rea1}) and hence $p_T$ would be large.
Conversely, if $\epsilon$ would be set too large, learning would be very easy ($m_T$ is very small),
but the resulting detection capability would be low.
So it is of interest to the user to choose $\epsilon$ carefully.

The gap between the 
practically observed number of mistakes $m_T$ in this scenario (in terms of $p_T$)
and the theoretical result (in term  of $\sigma_T$),  
prompts the need for further work.
Note that this is much similar to the state of affairs in studying the mistake bound 
in the good ol' perceptron, see e.g. \cite{freund99}. 

Since a sample can contribute at most a factor of size $\gamma_t^2$ to $\wv_T$,
the influence of a single point is bounded in the solution.
Since this influence is independent of the absolute size of the  sample,
the estimator is robust to outliers as in \cite{rousseeuw2005robust}.
This is a critical property in the context of fault detection where 
outlying points are present by definition.
As this property is violated in case of unknown $\epsilon$ (see below)
the FADO estimator detailed in Alg. (\ref{alg.fado2}) is less robust, 
also observed in the subsequent experiments.

\subsection{Realisable case, unknown $\epsilon>0$}
%%%%%%%%%%%%%%%%%%

Now we consider the case where  $\wv$ as well as $\epsilon>0$ are unknown.
The algorithm builds up a sequence of $\{(\wv_t,\epsilon_t)\}_t$ 
which can be thought of as approximating this $(\bar\wv,\epsilon)$.
However, $\epsilon_t$ is represented by its inverse such 
that the algorithm decides that $\yv_t$ is an atypical instance as 
\begin{equation}
	\frac{1}{\epsilon_{t-1}} \| \yv_t - \wv_{t-1}\| > 1. 
	\label{eq.rule}
\end{equation}
The key insight is to take 
\begin{equation} 
	\epsilon_t = \frac{1}{\gamma_t}, \ \forall t, 
	\label{eq.epst}
\end{equation}
and hence to let the radius $\epsilon_t$ of the detection mechanism grow slowly 
so that there is enough time to adjust $\wv_t$ if needed.
Hence, the decision variable is now
\begin{equation}
	d_t=
		\begin{cases}
			1 & \mbox{ \ If \ } \gamma_{t-1}\| \yv_t - \wv_{t-1}\|_2\geq 1\\
			0 & \mbox{ \ Otherwise.}	
		\end{cases}
		\label{eq.d3}
\end{equation}
The full algorithm is detailed in Alg. (\ref{alg.fado2}).

\begin{algorithm}
\caption{FADO}
\label{alg.fado2}
\begin{algorithmic}
\STATE Initialise $\wv_0=0_d$.
\FOR {$t=1,2,\dots$}
\STATE(1) Receive transaction $\yv_t\in\R^n$
\STATE(2) Raise an alarm (set $d_t=1$) if 
	\[ \gamma_{t-1} \left\| \yv_t - \wv_{t-1}\right\|_2 \geq 1, \]
	and set 
	\[ \vv_t = \frac{\yv_t - \wv_{t-1}}{\|\yv_t-\wv_{t-1}\|_2}\in\R^n.\]
\STATE(3) if so, update
	\[ \wv_t = \wv_{t-1} + \gamma_{t-1} \vv_t. \]
	If no alarm is raised ($d_t=0$), 
	set $\wv_t=\wv_{t-1}$.
\ENDFOR
\end{algorithmic}
\end{algorithm}

Now we proof that this method makes only a bounded number of mistakes, while 
having increasing detection capabilities.
Again we assume the existence of a perfect solution $(\bar\wv,\epsilon)$ with $\epsilon>0$
such that the following holds for all transactions,
\begin{equation}
	\|\bar\wv - \yv_t\|_2 <  \epsilon, \forall t.
	\label{eq.rea3}
\end{equation}
or that there exists a $\mu>0$ such that 
\begin{equation}
	\|\bar\wv - \yv_t\|_2 \leq \epsilon -\mu, \forall t.
	\label{eq.rea4}
\end{equation}
This {\em realisability} assumption is not restrictive as $\epsilon\geq 0$ 
can be arbitrary large.
However, if $\epsilon$ is too large, the rule $\|\bar\wv - \yv_t\|_2>\epsilon$ has no power anymore,
i.e. it will not detect interesting cases.
The smaller this $\epsilon$ is, the more powerful the rule is.

\begin{Theorem}
	Assume the existence of a couple $(\bar\wv,\epsilon)$ such that eq. (\ref{eq.rea1})
	holds for all transactions $\{\yv_t\}_t$, and assume that there 
	is a finite constant $\leq R<\infty$ such that $\max_t\|\yv_t\|_2 \leq R$.
	Let $\{\gamma_t>0\}_t$ as 
	 \begin{equation}
		\gamma_t = \frac{\gamma_0}{(m_t+1)^{1/2+\tau}}, \ \ \forall t\geq 0,
		\label{eq.gam2}
	\end{equation}
	and $\gamma_0>0$, then
	\begin{equation}
		m_t \leq O\left(\|\bar\wv\|_2^2\right).
		\label{eq.thm2}
	\end{equation}
\end{Theorem}

\begin{proof}
	Again, the recursion implemented by the algorithm can be stated concisely as 
	\begin{equation}
		\wv_t = \sum_{s=1}^t d_s \gamma_{s-1} \vv_s,
		\label{eq.proof4.1}
	\end{equation}	
	where $d_s=1$ in case $\gamma_{s-1}\|\yv_s -\wv_{s-1}\|_2>1$, and $d_s=0$ otherwise.
	Hence, since eq. (\ref{eq.rea3}), and 
	\begin{equation}
		 \vv_s^T(\bar\wv - \yv_s) \geq -\epsilon,
		\label{eq.proof4.2}
	\end{equation}	
	or $ \vv_s^T\bar\wv \geq \vv_s^T\yv_s -\epsilon$, then one has 
	\begin{equation}
		\bar\wv^T\wv_t 
		= \sum_{s=1}^t d_s \gamma_{s-1} \vv_s^T\bar\wv
		\geq \sum_{s=1}^t d_s \gamma_{s-1} (\yv_s^T\vv_s  - \epsilon).
		\label{eq.proof4.3}
	\end{equation}	
	Then we have in case $d_t=1$ where 
	$\gamma_{t-1}\| \yv_t - \wv_{t-1}\|_2\geq 1$, that
	\begin{equation}
		\gamma_{t-1}\vv_t^T(\yv_t - \wv_{t-1}) 
		=
		\gamma_{t-1}\| \yv_t - \wv_{t-1}\|_2 
		\geq 
		1,
		\label{eq.proof4.4}
	\end{equation}	
	and 
	\begin{equation}
		\sum_{s=1}^t d_s  (\gamma_{s-1}\vv_s^T\yv_s  -  \gamma_{s-1}\epsilon)
		\geq 
		\sum_{s=1}^t  d_s \gamma_{s-1} (\vv_s^T\wv_{s-1} -\epsilon) + m_t.
		\label{eq.proof4.5}
	\end{equation}	
	Conversely, since from Cauchy-Schwarz' inequality follows that 
	$\bar\wv^T\wv_t\leq \|\bar\wv\|_2 \|\wv_t\|_2$.
	In instances $t>0$ where $d_t=1$, one has 
	\begin{equation}
		\wv_t^T\wv_t 
		= \|\wv_{t-1}\|_2^2 + \gamma_{t-1}^2\|\vv_t\|_2^2 
		+ 2\gamma_{t-1}\wv_{t-1}^T\vv_t \\
		= 2 \sum_{s=1}^t d_s \gamma_{s-1}\wv_{s-1}^T\vv_s
		+  \sum_{s=1}^t d_s \gamma_{s-1}^2,
		\label{eq.proof4.6}
	\end{equation}	
	as in case $d_s=0$, one has $\|\wv_{s}\|_2^2=\|\wv_{s-1}\|_2^2$.
	Combining inequalities and using that eq. (\ref{eq.proof4.6}) is always larger than 0, gives 
	\begin{multline}
		\left(m_t - \sum_{s=1}^t  d_s \gamma_{s-1} \epsilon\right)
		+ \left(\sum_{s=1}^t  d_s \gamma_{s-1} \wv_{s-1}^T\vv_s\right)  \\
		\leq 
		\|\bar\wv\|_2 \sqrt{ 2 \left(\sum_{s=1}^t d_s \gamma_{s-1}\wv_{s-1}^T\vv_s\right)
		+  \left(\sum_{s=1}^t d_s \gamma_{s-1}^2\right) }.
		\label{eq.proof4.7}
	\end{multline}	
	Then using the AC-Lemma gives 
	\begin{equation}
		\left(m_t - \sum_{s=1}^t  d_s \gamma_{s-1} \epsilon\right) \\
				\leq 
		O\left(\|\bar\wv\|_2^2\right),
		\label{eq.proof4.7b}
	\end{equation}	
	in case $t$ is large enough so that $m_t > \sum_{s=1}^t  d_s \gamma_{s-1} \epsilon$.
	Now we look at the term 
	\begin{equation}
		\left(m_t - \sum_{s=1}^t  d_s \gamma_{s-1} \epsilon\right)
		=
		\sum_{s=1}^t  d_s \left(1 -  \gamma_{s-1} \epsilon\right).
		\label{eq.proof4.8}
	\end{equation}	
	In case $\gamma_s\epsilon>0$, this results in negative contributions.
	However, once $\gamma_s>0$ is small enough, the summands are all positive and 
	the previously derived upper-bound becomes active.	
	Formally, this is translated into the following lower-bound 	
	\begin{multline}
		\left(m_t - \sum_{s=1}^t  d_s \gamma_{s-1} \epsilon\right)
		= 
		\left(m_t -  \sum_{s=1}^{s'}  d_s \gamma_{s-1} \epsilon
		- \sum_{s=s'+1}^{t}  d_s \gamma_{s-1} \epsilon\right) \\
		= 
		\left(\sum_{s=1}^{s'}  d_s (1-\gamma_{s-1} \epsilon)
		+ \sum_{s=s'+1}^{t}  d_s (1-\gamma_{s-1} \epsilon)\right) \\
		\geq 
		\left( (1-\gamma_0\epsilon) \sum_{s=1}^{s'}  d_s 
		+(1-\gamma_{s'} \epsilon) \sum_{s=s'+1}^{t}  d_s \right),
		\label{eq.proof4.10}
	\end{multline}	
	for any $1\leq s'<t$.
	Then changing indexing from $s'$ to $1\leq m'\leq m_t$ gives 
	\begin{multline}
		\left( (1-\gamma_0\epsilon) \sum_{s=1}^{s'}  d_s 
		+(1-\gamma_{s'} \epsilon) \sum_{s=s'+1}^{t}  d_s \right) \\
		\geq 
		(1-\gamma_0\epsilon) m' 
		+ \left(1-\frac{\gamma_0\epsilon}{{m'}^{\frac{1+2\tau}{2}}}\right) (m_t-m') \\
		=
		m_t - \gamma_0\epsilon m' +
		 \left(\frac{-\gamma_0\epsilon}{{m'}^{\frac{1+2\tau}{2}}}\right) (m_t-m') \\
		= m_t - \gamma_0\epsilon  
			\left(m' + \frac{(m_t-m')}{{m'}^{\frac{1+2\tau}{2}}}\right). 
		\label{eq.proof4.11}
	\end{multline}	
	Since this holds for any $0<m'\leq m_t$, one can choose (solve) $m'$ such that
	that
	\begin{equation}
		\left(m' + \frac{m_t-m'}{{m'}^{1/2+\tau}}\right) = \frac{m_t}{2\epsilon\gamma_0},
		\label{eq.proof4.12}
	\end{equation}	
	such that 
	\begin{equation}
		\left(m_t - \sum_{s=1}^t  d_s \gamma_{s-1} \epsilon\right)  
		\geq
		\frac{1}{2} m_t,
		\label{eq.proof4.12}
	\end{equation}	
	yielding the result.
\end{proof}

\section{Numerical experiments}
%%%%%%%%%%%%%%%%%

This section presents result substantiating practical efficacy of 
the FADO algorithm under different scenarios:
this section investigates whether the previously derived theoretical results 
agree with the practical results in a controlled, artificially constructed case.

The experiment has the following setup.
Again, $T=10000$ `normal' samples are generated 
centered around $\bar\wv=c(1, \dots, 1)^T\in\R^n$ for a constant $c>0$.
All those lie uniformly spread within radius $1-\mu$ of this centre.
Then the abnormal samples are drawn from a much wider distribution, 
excluding the one that lie within a radius 1 of the centre.
Then the FADO($\epsilon=1$) algorithm was run on this dataset, and the 
number of false alarms $m$ was recorded.

\begin{figure}[htbp]
	\begin{center}\includegraphics[width=0.75\textwidth]{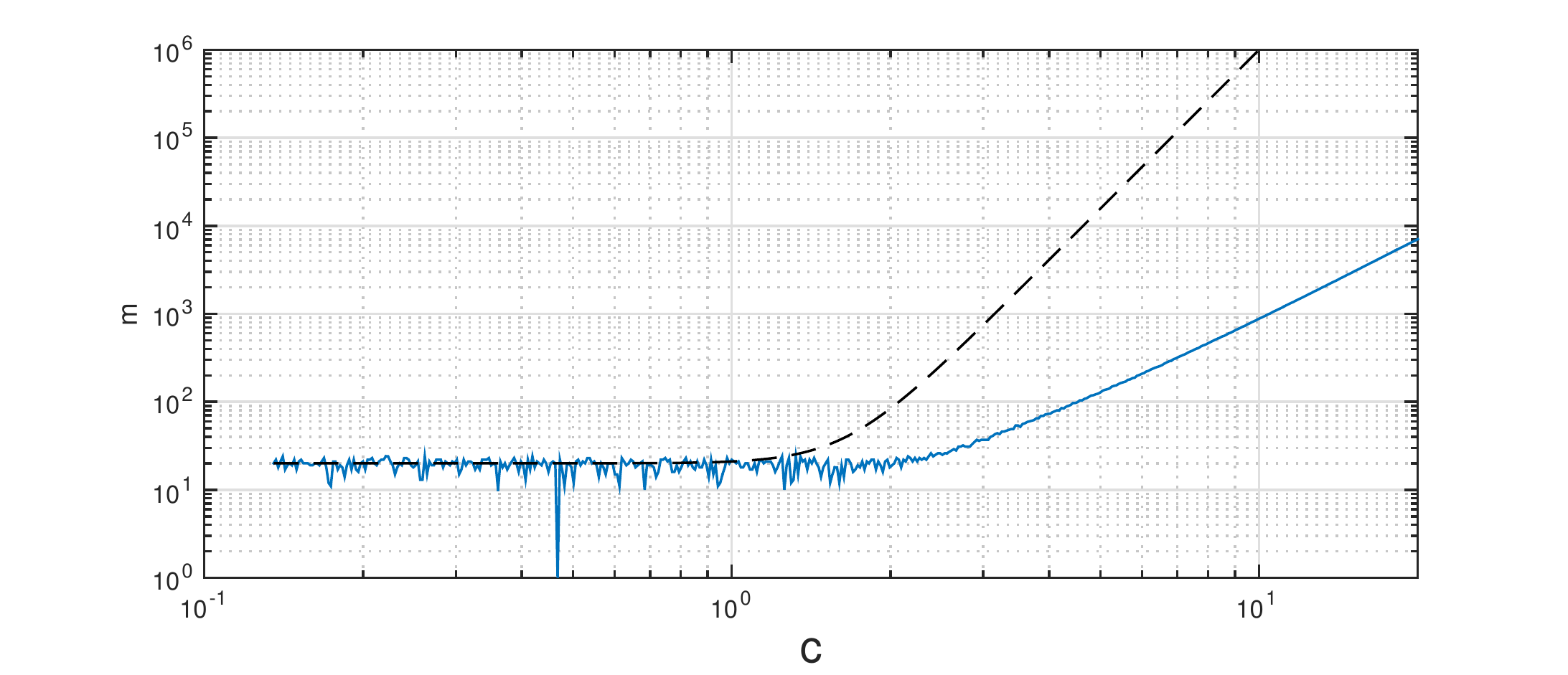}\end{center}
	\caption{\em 
		Number of alarms $m_T$ (Y-ax) in terms of different values of $c$ (X-ax)
		where $\bar\wv=c 1_n$.		
		Blue solid line: empirical numbers, black dashed line: theoretical bound of $O(c^4)$
		(modulo constants).}
	\label{n1}
\end{figure}

\begin{figure}[htbp]
	\begin{center}\includegraphics[width=0.75\textwidth]{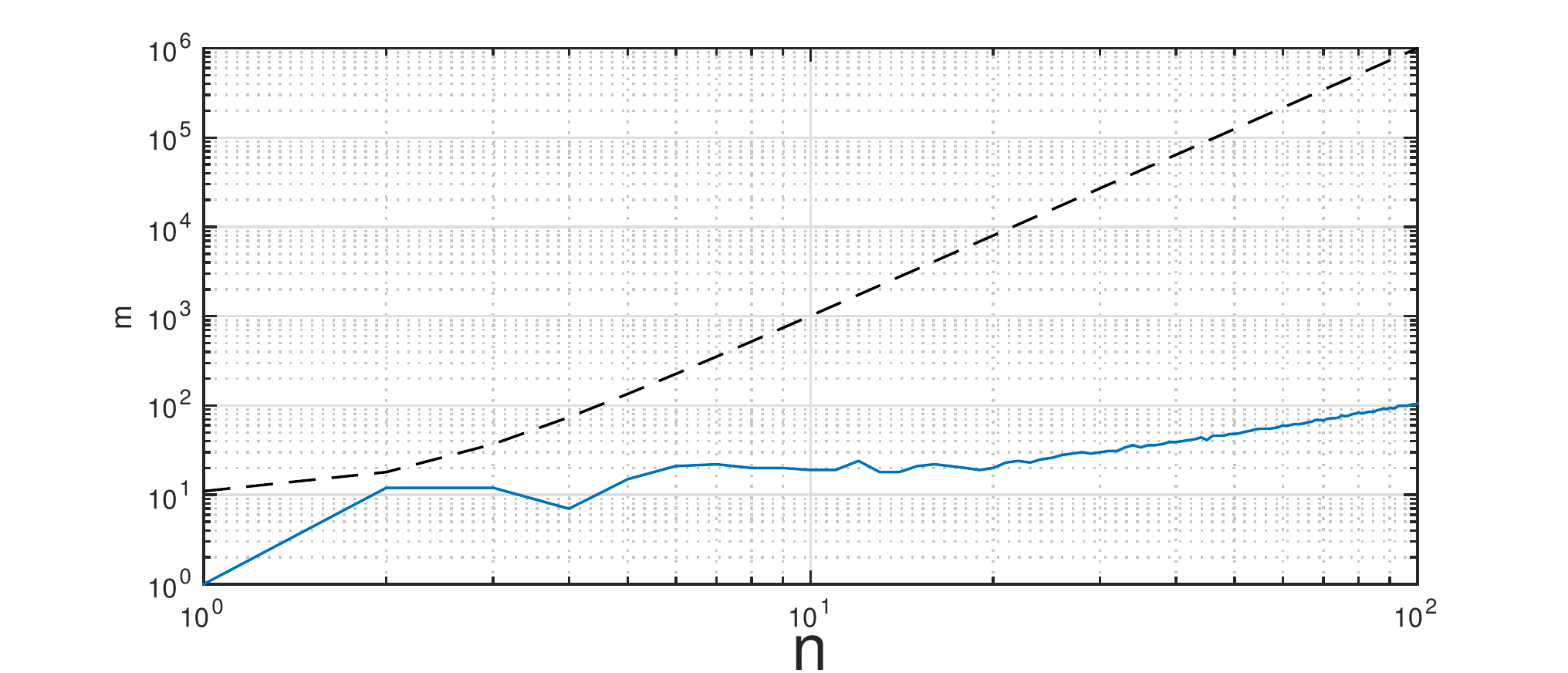}\end{center}
	\caption{\em 
		Number of alarms $m_T$ (Y-ax) in terms of different values of the dimension 
		$n$ (X-ax).		
		Blue solid line: empirical numbers, black dashed line: theoretical bound of $O(n^4)$
		(modulo constants).}
	\label{n2}
\end{figure}

\begin{figure}[htbp]
	\begin{center}\includegraphics[width=0.75\textwidth]{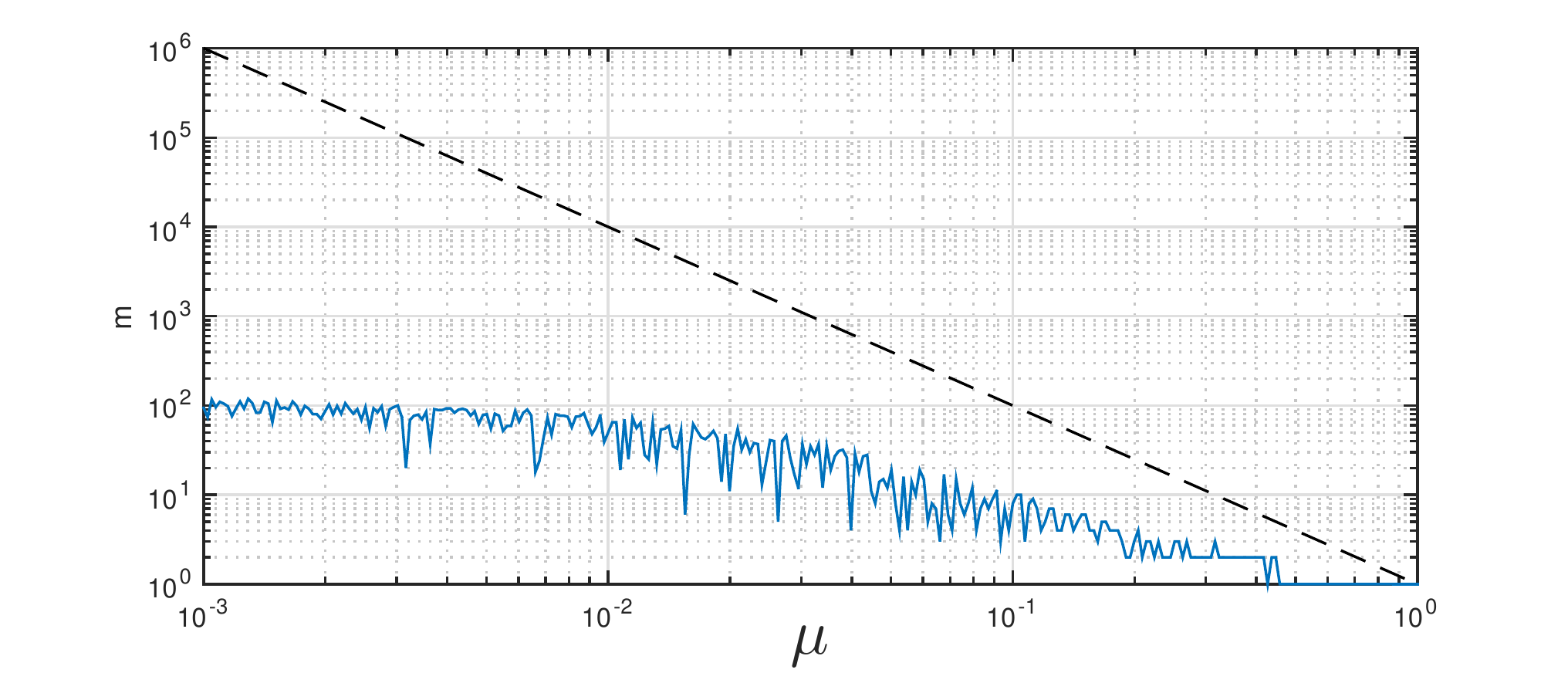}\end{center}
	\caption{\em 
		Number of alarms $m_T$ (Y-ax) in terms of different values of $\mu$ (X-ax).		
		Blue solid line: empirical numbers, black dashed line: theoretical bound  of $O(\mu^{-2})$
		(modulo constants).}
	\label{n3}
\end{figure}

\begin{figure}[htbp]
	\begin{center}\includegraphics[width=0.75\textwidth]{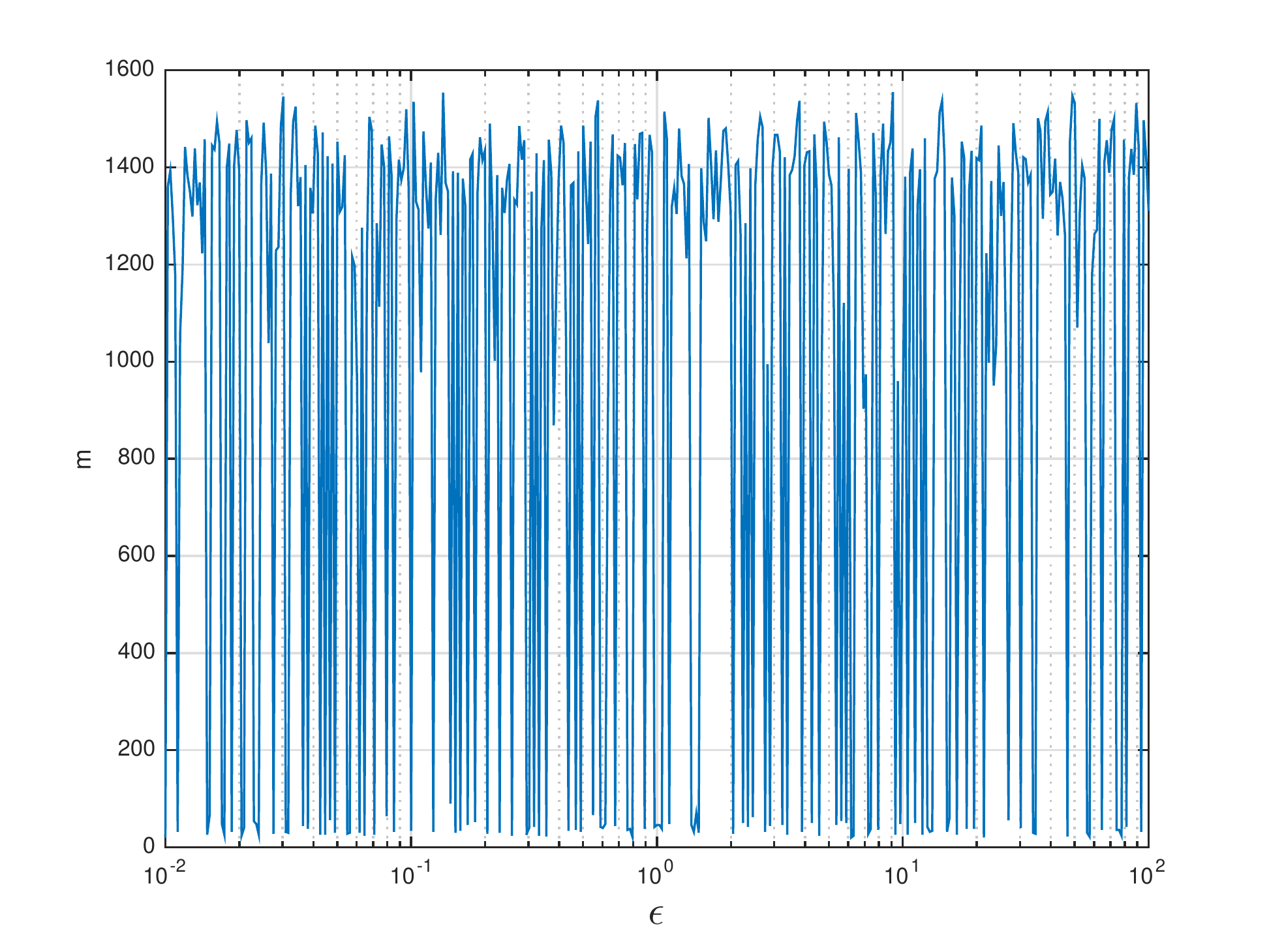}\end{center}
	\caption{\em 
		Number of alarms $m_T$ (Y-ax) in terms of different values of $\epsilon$ (X-ax)
		ranging from $10^{-2}$ to $10^{2}$. The values for $\mu,c$ and $n$ are randomised.
		Note that no trend is observed in the values of $m_T$.
		}
		\label{n4}
\end{figure}

Close inspection of Figures (\ref{n1}), (\ref{n2}), (\ref{n3}) and (\ref{n4}) 
evidences a number of insights for the FADO algorithm with fixed $\epsilon=1$:
\begin{itemize}
\item The number of mistakes $m$ varies (inversely) proportional the quantities $\mu,c,n$ as 
	presented in the bound. 
	
\item The precise orders seem to be overly conservative in the bound. 
	For example, figure \ref{n3} indicates that $m$ behaves inversely proportional to $\mu$.
	However, the bound of theorem 1 derives a dependency on $\mu$ as $\mu^{-2}$.
	
\item Note that especially the dependence on dimensionality seems very mild in
	practice, and the influence of $\|\bar\wv\|_2$ (here parametrised as $c$, see Fig. (\ref{n1})) seems to be royally overestimated.
	Indeed, we see that the algorithm is very well equipped to work in high dimensions.
	Here, for example we see that $n=100$ only incurs $m=\pm 100$ mistake alarms 
	(that is, a false alarm is raised in only $\pm 0.1\%$  of the presented cases.
	
\item The number of alarms $m$ does not vary with the radius $\epsilon$, nor 
	with the size of the data-points $\max_t\|\yv_t\|_2$. This is surprising as it is unlike other 
	mistake-based algorithms where such normalisations are explicitly present in 
	bounds and in the practically observed performances, see e.g. \cite{kivinen1995perceptron}.
\end{itemize} 
This then evidences that the algorithm works in general even better than as indicated by the 
presented (worst-case) bounds. 
This is  positive news towards applicability, but does not imply that the bounds are not 
tight as there might be really hard cases.
Note that the derivation of lower bounds for mistake-driven algorithms - often based 
on Hadamard constructions - or see e.g. \cite{kivinen1995perceptron} is scarcely addressed.

Lastly, we assess the performance of the FADO algorithm (unknown $\epsilon$) in Fig. (\ref{n32}).
As argued before, the more agnostic nature of this algorithm feeds on power, and 
is less robust. While the second item is established before, we assess power in case 
of $n=2$, $\|\bar\wv\|_2=\sqrt{2}$, and for a range of values $\mu$.
From the results, we see that the power flatlines around $98\%$ for most cases,
while it was consistently $100\%$ in case $\epsilon$ was provided.
Also note that the theoretical bound of $O(1/\mu^2)$ is not very tight in this case.

\begin{figure}[htbp]
	\begin{center}\includegraphics[width=0.75\textwidth]{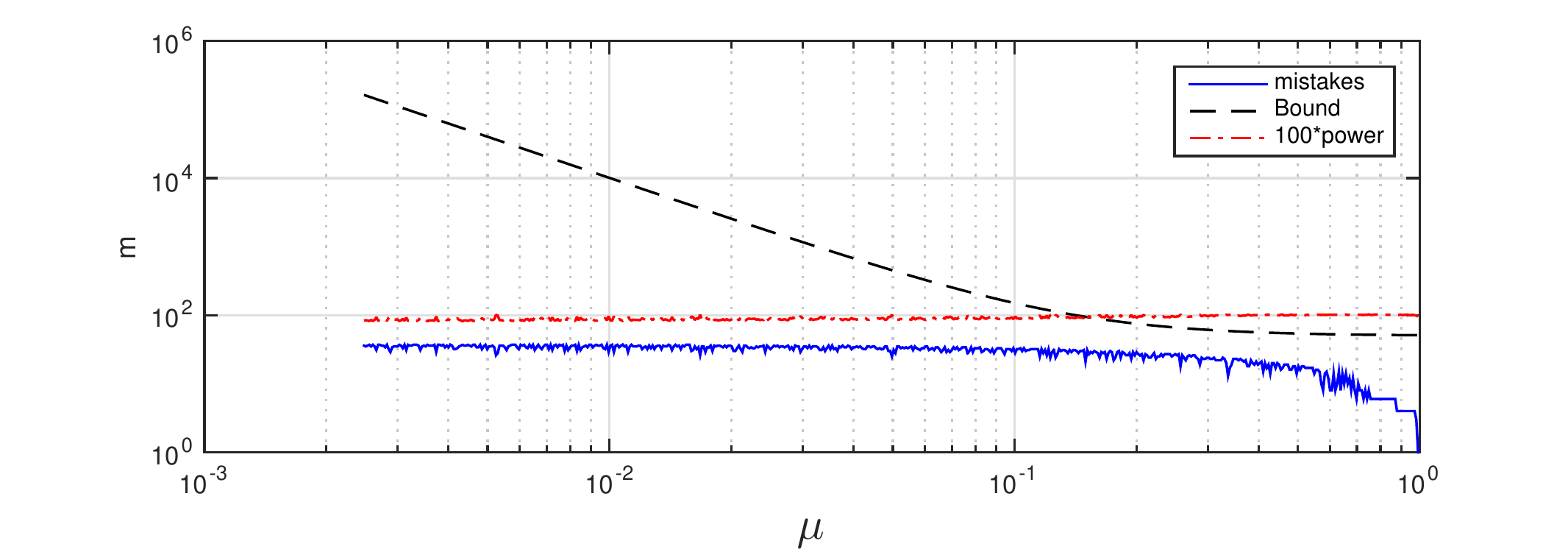}\end{center}
	\caption{\em 
		Number of alarms $m_T$ (Y-ax) in terms of different values of $\mu$ (X-ax)
		for the FADO algorithm where $\epsilon$ is unknown.		
		Blue solid line: empirical numbers, 
		Black dashed line: theoretical bound (modulo constants),
		The red dashed-dotted line gives {\em as a reference} the performance of the FADO($\epsilon$) algorithm where $\epsilon$ is given . }
	\label{n32}
\end{figure}

\section{Case study}
%%%%%%%%%%%

This section presents results obtained in an application of video processing for detecting transitions between scenes.
This case study is challenging for: 
\begin{itemize}
\item Its inherent high-dimensionality, $n=O(10^5)$, 
	but (relatively) small number of samples $T=O(10^4)$,
\item Its time-varying nature ('tracking'),
\item The presence of outliers, structural artefacts and typical video-based distortions, and 
\item The presence of intricate (but unknown) dependency structures, encoding the {\em meaning} of the video.
\end{itemize}
The aim of this case study is not so much 
to relate the performance of FADO to alternative approaches using 
advanced image processing, bounding box estimation and object recognition,
but rather to indicate how the unaltered FADO algorithm based on a fairly naive model 
does already surprisingly well using an absolute minimum of computing resources.
For a survey of applicable image processing tools for such task (also called {\em cut} 
or {\em shot transition detection}), see e.g. \cite{koprinska2001temporal} and subsequent work.
Note that one can deploy many clustering-based approaches for this task, but that 
the non-stochastic nature of the application invalidates the usual theoretical/stochastic results.
The principal objective of this case study is to illustrate that the present framework/algorithm does however provides a framework for handling such case.

\subsection{Case study: Setup}

A video\footnote{The complete video of the case study can be downloaded from 
\url{http://user.it.uu.se/~kripe367} $\rightarrow$ FADO.}
 is constructed by piecing together 16 short animal motion movie clips.
The original video clips are taken from the pioneering work of E.Muybridge (1830-1904).
The task now is to detect frames in the movie where transition to a subsequent clips are made
(the clips are depicting 
'ostrich', 'horse', 'bison', kangaroo', 'deer', 'leopard', 'dog', 'pig', 'donkey' and 'elephant').
For that, one needs to represent the subjects of the present clip.
One may do this via advanced image processing techniques, but inline with the 
previously developed theory, the $i$th clip is represented as a couple $(\bar\wv_i,\epsilon)$
where the frames $\yv_t$ of the $i$th clip obey
\begin{equation}
	\left\|\yv_t - \bar\wv_i\right\|_2 < \epsilon.
	\label{eq.case.eps}
\end{equation}
Here, the $t$th frame is represented as a vector $\yv_t\in[0,1]^n$ with $n=400\times 400=16.10^4$.
Traditional dimensionality reduction or manifold learning as in \cite{hastie01} 
are not really an option in the present case as the eigenvalue spectrum of those frames $\{\yv_t\}_t$ 
decays too slowly to allow for linear Principal Component Analysis (PCA), 
the enormous dimension hinders the application of advanced density/covariance estimation techniques,
while the manifold (based on the movement of the animals) changes too drastically over the different clips
for manifold learning to be applicable.
Moreover, since such approaches needs one to perform a dimensionality reduction phase on a dedicated set of data, 
one suffers a consequent loss of detection power,
 while it is not clear how (theoretically) to choose such dimensionality optimally.

After proper rescaling of the frames, one obtains the video as presented, 
consisting of  $T=3330$ different frames.
Note that this video is far from {\em clean}, including artefacts of the photography, the mouse pointer moving through the screen, and even a ghost window flashing up.
The corresponding dataset $\{\yv_t\}_{t=1}^{3330}$ with $n=160.000$ was however submitted to the FADO($\epsilon$) algorithm without any further preprocessing.

Choice of the threshold $\epsilon\geq 0$ in eq. (\ref{eq.case.eps}) is crucial to proper operation, and 
translates the intuition {\em 'how much can a scenery vary before it can be called a next clip.'}
Note that this $\epsilon$ can be learned from the data as well (see Subsection 2.4), 
but for the price of a seizable loss of robustness and power. 
Since this application is challenging enough as is, 
this threshold was fixed to $\epsilon=100$, based on simple, intuitive 
experiments of a single clip ('ostrich'). 

In this application, it is reasonable to expect that each different clip has a different vector $\bar\wv_i$
associated to it. 
This makes the setting time-varying. For such {\em tracking} settings, it is argued (see e.g. \cite{ljung1990,herbster2001tracking})
to take the gain as a constant. 
In the present case, we set this to unit, i.e. 
\begin{equation}
	\gamma_t = 1, \ \forall t.
	\label{eq.case.eps}
\end{equation}
Besides this, no tuning was required for this application.

\subsection{Case study: Discussion}
%%%%%%%%%%%%%%%%%%%

Figures (\ref{mb1}), (\ref{mb2}), (\ref{mb3}) and (\ref{mb4}) display 
$\yv_t$ and $\wv_{t-1}$ for  $t=99, 650, 1500$ and $t=3330$ respectively.
Figure (\ref{mb5}) gives the changepoints between the 16 clips (solid line), 
and the time-points where FADO(100) gives a detection/update.
The FADO(100) detects/updates on $\frac{1047}{3330} = \pm 34\%$ of the frames.
Since 16 clips were joined together, only $\frac{15}{3330} = \pm 0.5\%$ scenery changes (true alarms) exists.
From this experiment, a number of insights are gained.
\begin{itemize}
\item (Simple model)
	First of all, a `simple rule' as detailed in eq. (\ref{eq.case.eps})
	does already a surprising good job, despite the fact that we did not resort to 
	any image processing tools to {\em interpret} the content of the video.
	This is especially useful in `agnostic' cases where content is not known (hardcoded)
	in advance, i.e. where supervised approaches cannot be applied.
	
\item (Mistake-based learning)
	Observe that the $34\%$ of cases where FADO(100) did detect/update is not quite close to the 
	$0.5\%$ of cases where an actual transition did happen.
	However, figure  (\ref{mb5}) makes clear that the FADO(100) algorithm tags a series of frames after the transition happens.
	This is unavoidable as the algorithm needs to re-learn the new typical behavior of a new clip.
	A more fair comparison would be that 20-30 blocks of frames are associated to a changing scene, resulting in a 1\% detection rate.
	However, note that when $t$ progresses, 
	this re-adjusting phase becomes shorter and $\wv_t$ becomes more generic.
	This can also be seen by close inspections of subplots (b) 
	of Fig (\ref{mb1}), (\ref{mb2}), (\ref{mb3}) and (\ref{mb4}). 
	This can be intuitively interpreted that the algorithm {\em memorizes}  more and more 
	complex (mixtures of) patterns.

\item (Learning rate)
	Another issue where this application supplements the theoretical result is that 
	in this application $\bar\wv$ changes over time (`tracking'), and a constant gain $\gamma_t$
	 as in eq. (\ref{eq.case.eps}) is used, rather than the decreasing gain as e.g. in eq. (\ref{eq.a.gam}).
	The choice of a constant gain in this setting is quite intuitive because 
	the algorithm cannot converge to a single solution
	but needs a seizable mechanism to keep updating. 
	Note that the algorithm has to learn `more' (i.e. makes 
	more false detections) in the beginning of the experiment (`burn-in phase').
	In later phases, the memory becomes more generic as different scenarios are seen,
	making re-adjustment easier.

\end{itemize}
Figure (\ref{mb6}) shows how the frames are organised in terms of their 3 first principal 
components.

\begin{figure}[htbp]
	\begin{center}\includegraphics[width=0.75\textwidth]{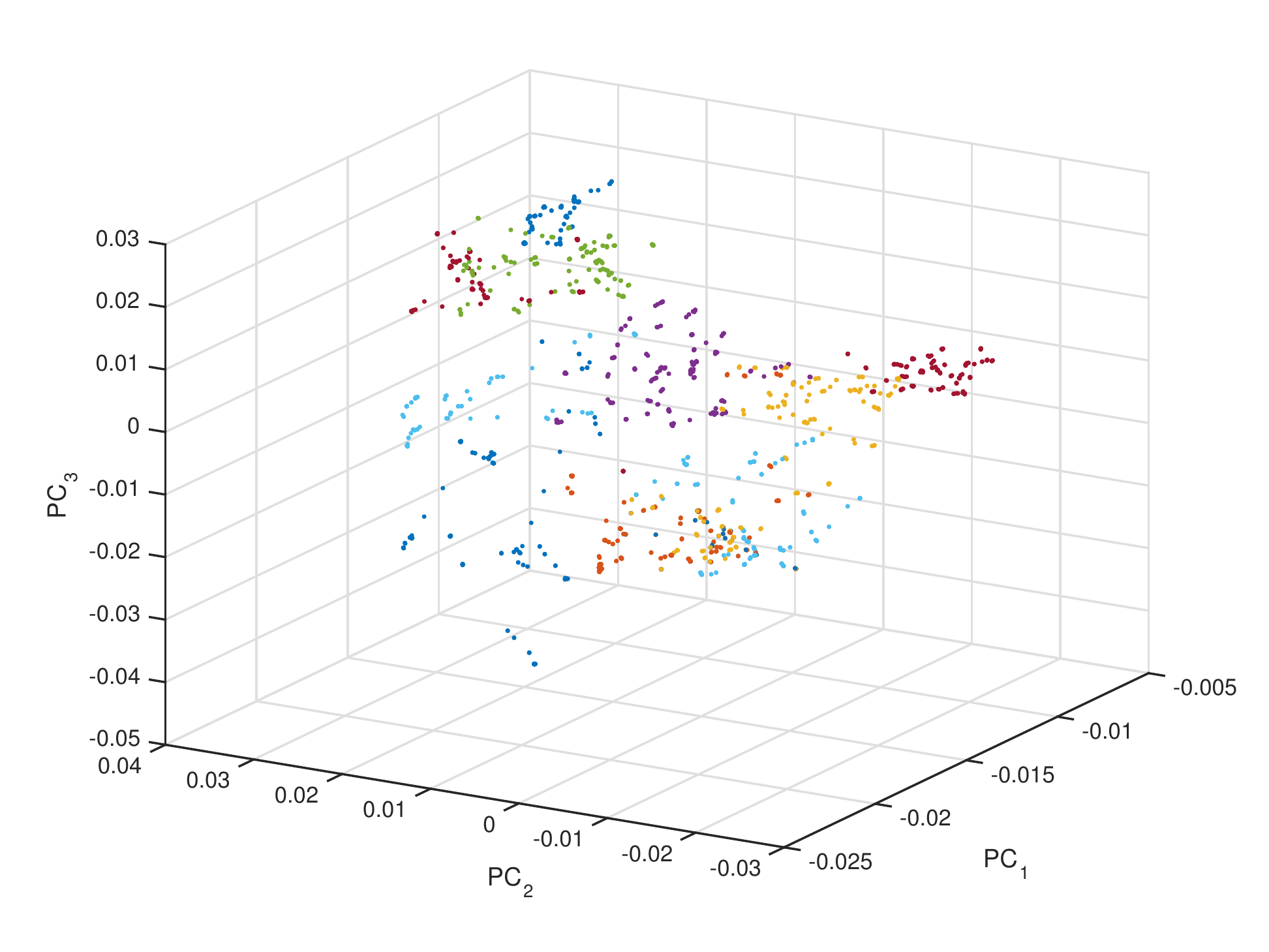}\end{center}
	\caption{\em All the frames $\{\yv_t\}_{t=1}^{3330}$ plotted as (`.')
	with color associated to 16 clips, in terms of the 3 largest Principal Components (PC) 
	of a (linear) PC Analysis (PCA): (PC$_1$,PC$_2$,PC$_3$).
	}
	\label{mb6}
\end{figure}

\begin{figure}[htbp]
	\begin{center}\includegraphics[width=0.75\textwidth]{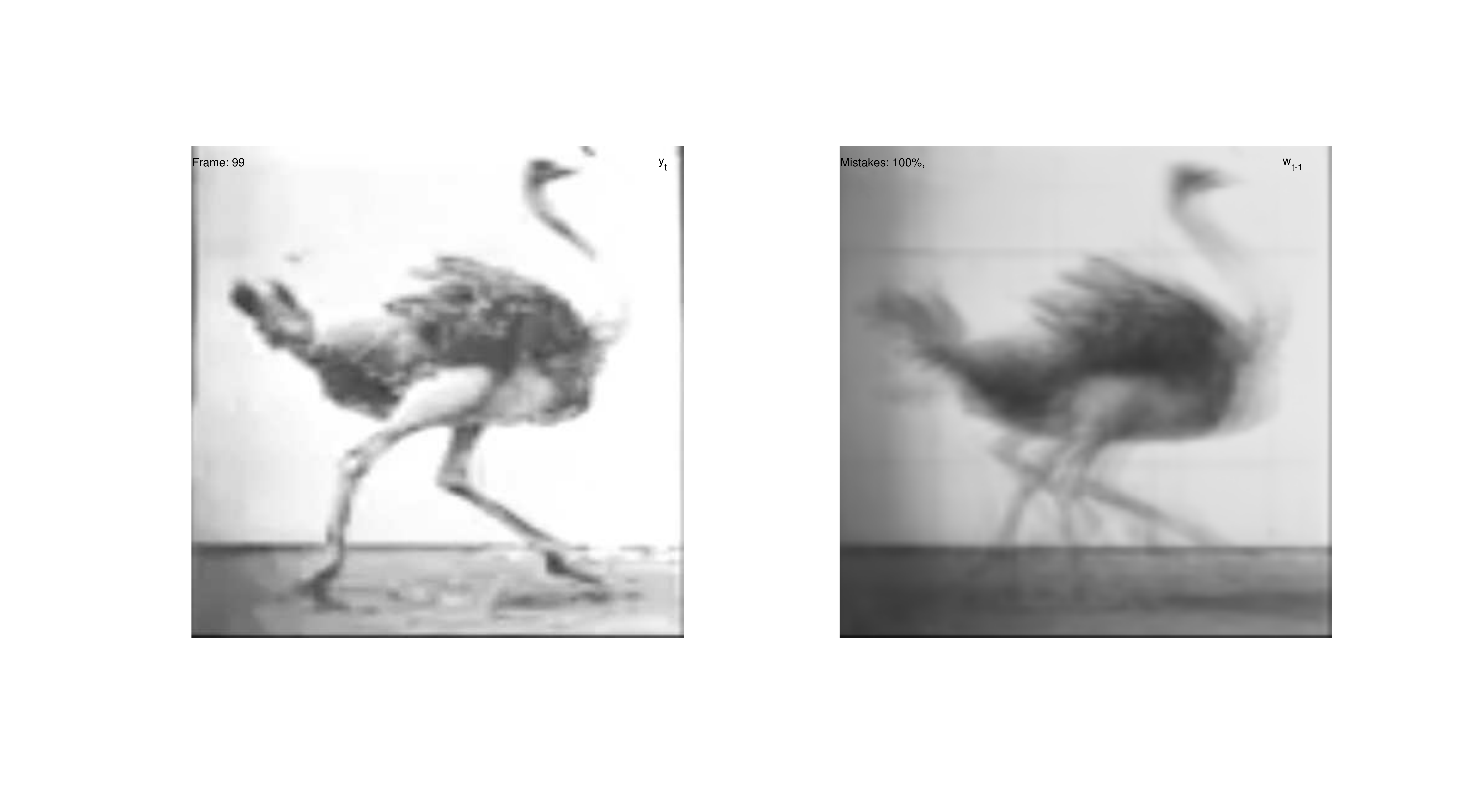}\end{center}
	\caption{\em 
	The result of FADO(100) on the Muybridge video with $n=400*400$ after $t=99$ iterations 
	(after having seen almost a single clip).
	Subplot (a) displays the frame $\yv_{99}$. 	Subplot (b) displays the 'memory' $\wv_{98}$.}
	\label{mb1}
\end{figure}

\begin{figure}[htbp]
	\begin{center}\includegraphics[width=0.75\textwidth]{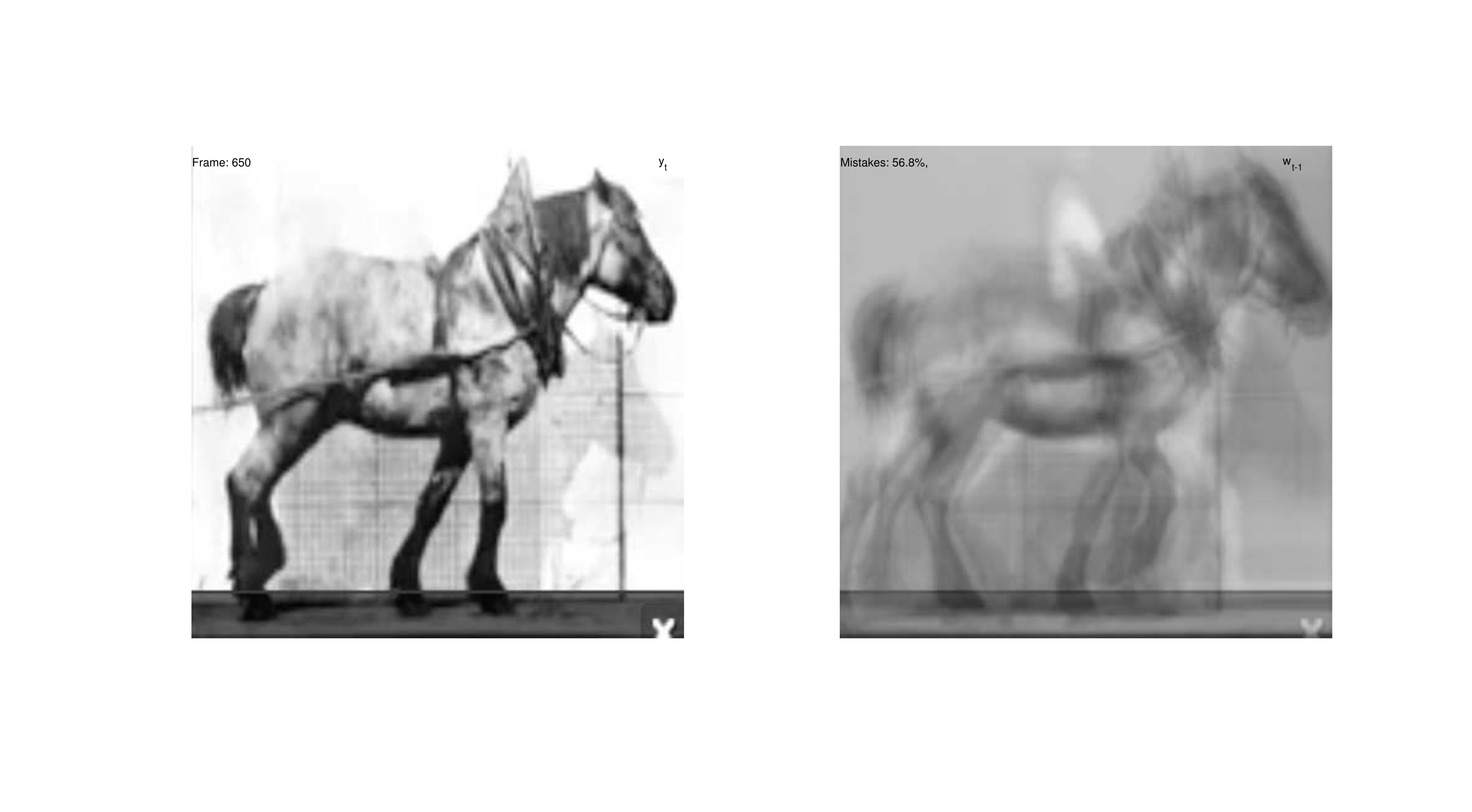}\end{center}
	\caption{\em 
	The result of FADO(100) on the Muybridge video after $t=650$ iterations 
	(after having seen almost 4 clips).
	Subplot (a) displays the frame $\yv_{650}$. 	Subplot (b) displays the 'memory' $\wv_{649}$.}
	\label{mb2}
\end{figure}

\begin{figure}[htbp]
	\begin{center}\includegraphics[width=0.75\textwidth]{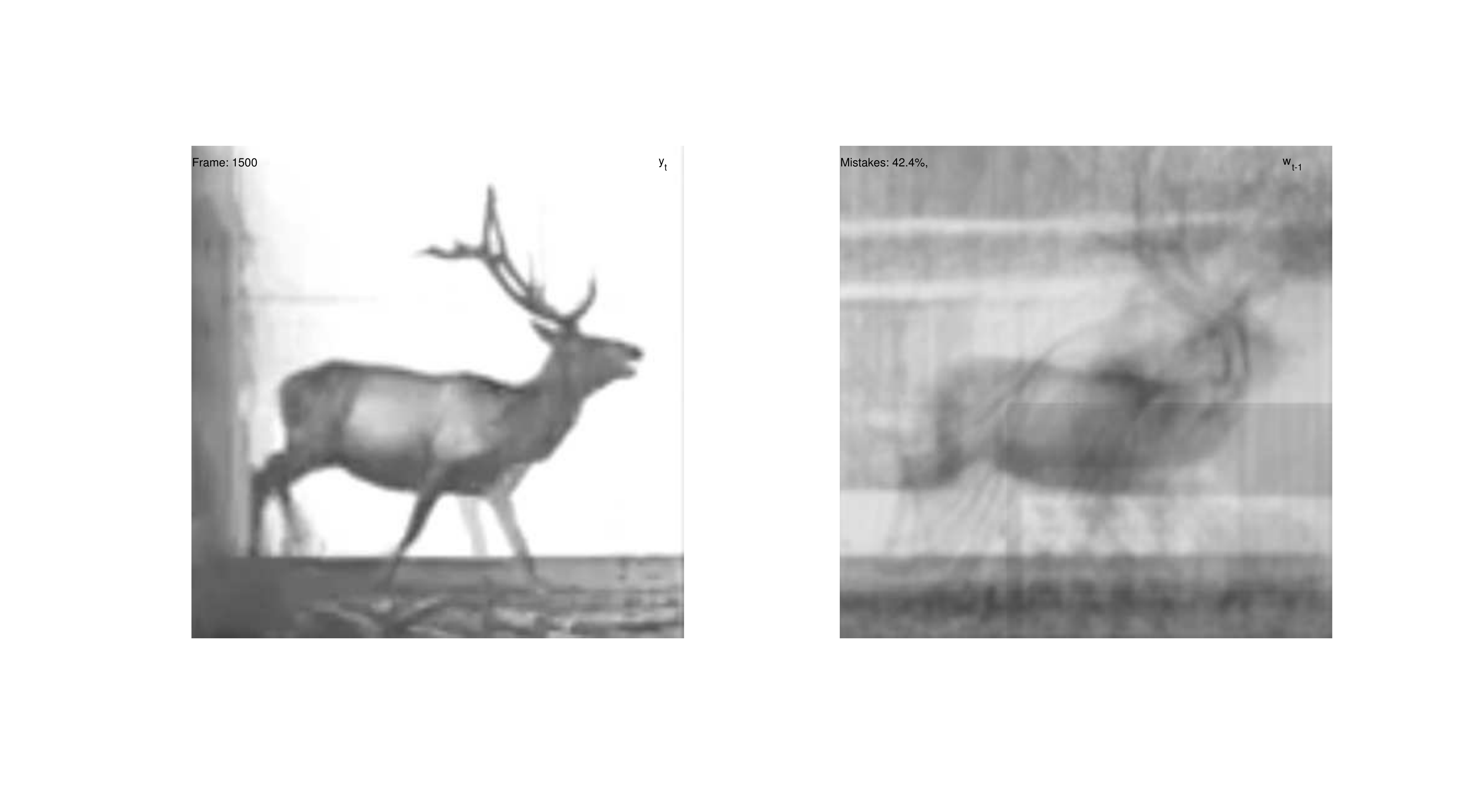}\end{center}
	\caption{\em 
	The result of FADO(100) after $t=1500$ iterations  
	(after having seen almost 7 clips).
	Subplot (a) displays the frame $\yv_{1500}$. 	Subplot (b) displays the 'memory' $\wv_{1499}$.}
	\label{mb3}
\end{figure}

\begin{figure}[htbp]
	\begin{center}\includegraphics[width=0.75\textwidth]{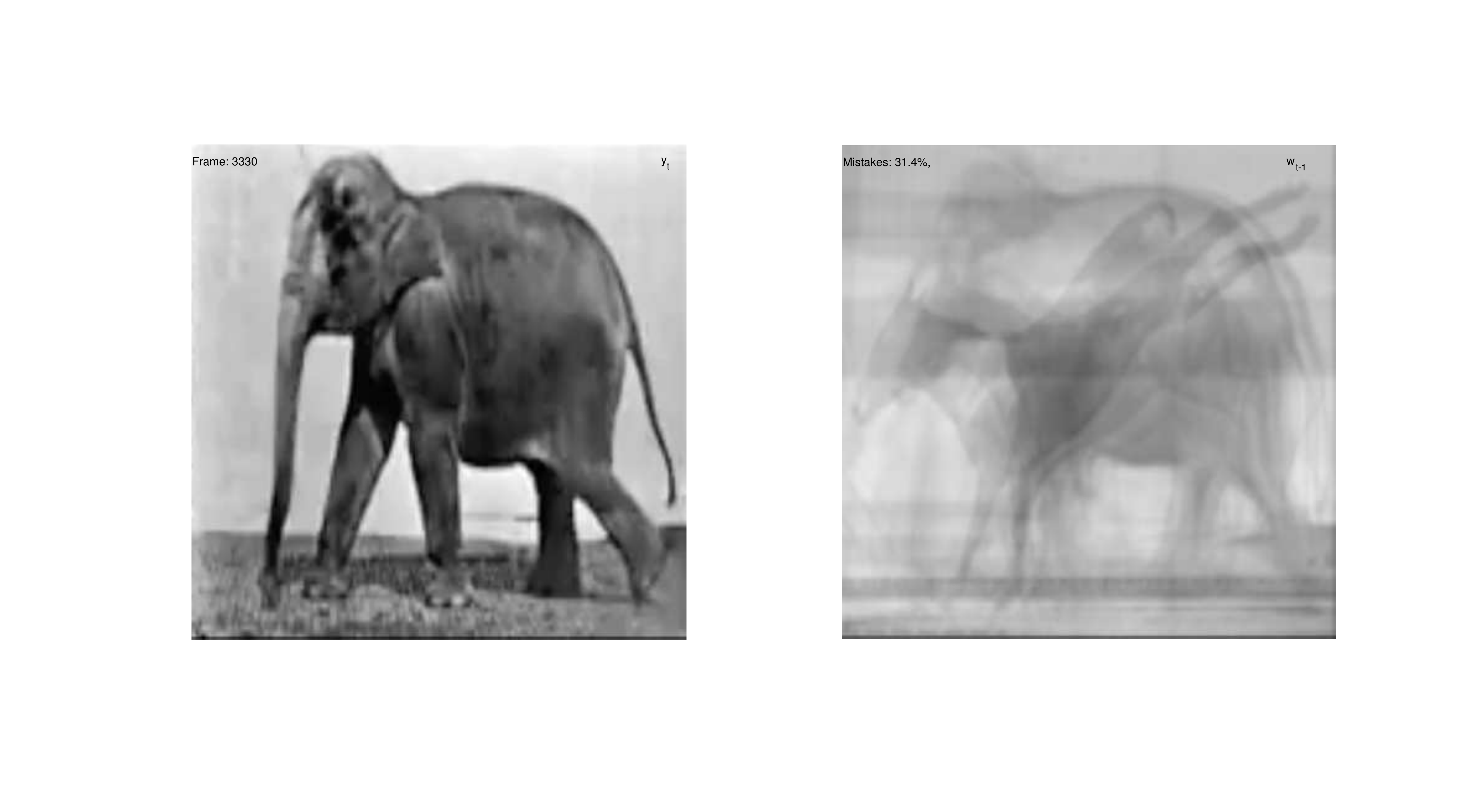}\end{center}
	\caption{\em 
	The result of FADO(100) after $t=3330$ iterations  
	(after having seen all 16 clips).
	Subplot (a) displays the frame $\yv_{3330}$. 	Subplot (b) displays the 'memory' $\wv_{3329}$.}
	\label{mb4}
\end{figure}

\begin{figure}[htbp]
	\begin{center}\includegraphics[width=1.1\textwidth]{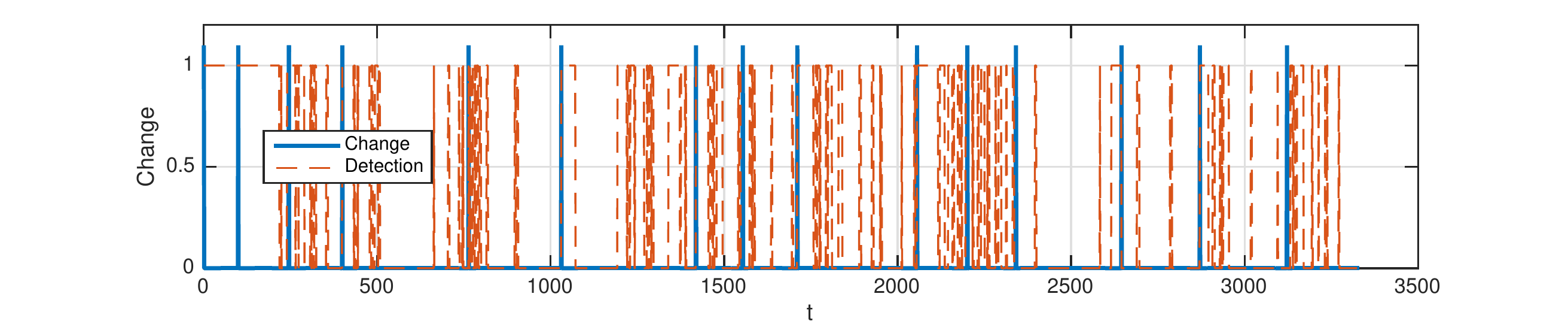}\end{center}
	\caption{\em 
	The result of FADO(100).
	The changepoints between the 16 different clips are shown as a blue solid line 
	(`1'=changepoint, `0'=contination).
	The red dashed line shows the detections 
	(`1': detection, `0': no detection). made by the application of FADO(100).}
	\label{mb5}
\end{figure}

\section{Conclusions}
%%%%%%%%%%%

This paper discusses a simple online (recursive) detection algorithm dubbed FADO 
for use in deterministic, adversary scenarios.
That is, for use in cases where methods of density estimation or stochastic inference are 
not applicable for their high-dimensionality, complex dependency structures,
presence of outliers, structural artefacts and other aberrations.
The main result is a theoretical foundation for use of (online) 
clustering methods for detection and data-mining applications.
Theoretical results are backed up by a numerical experiments and a video-processing case study.
The surprising bottomline is that no complex structures 
(as artificial neural networks or convolutional neural nets) are needed really.
This approach opens up for many new directions: 
e.g. how to start analysing more complex, multiple clustering methods? 
From a theoretical perspective, we like to see further work on tracking properties of
such mistake-based algorithms,
and explain the observed capability of the FADO algorithm to handle cases with enormous dimensionsionality.

%\bibliographystyle{alpha}
%\bibliography{../refs}

\begin{thebibliography}{25}
\providecommand{\natexlab}[1]{#1}
\providecommand{\url}[1]{\texttt{#1}}
\expandafter\ifx\csname urlstyle\endcsname\relax
  \providecommand{\doi}[1]{doi: #1}\else
  \providecommand{\doi}{doi: \begingroup \urlstyle{rm}\Url}\fi

\bibitem{awasthi2014center}
Pranjal Awasthi and Maria-Florina Balcan.
\newblock Center based clustering: A foundational perspective.
\newblock 2014.

\bibitem{ben2006sober}
Shai Ben-David, Ulrike Von~Luxburg, and D{\'a}vid P{\'a}l.
\newblock A sober look at clustering stability.
\newblock In \emph{International Conference on Computational Learning Theory},
  pages 5--19. Springer, 2006.

\bibitem{benjamini1995}
Yoav Benjamini and Yosef Hochberg.
\newblock Controlling the false discovery rate: a practical and powerful
  approach to multiple testing.
\newblock \emph{Journal of the Royal Statistical Society. Series B
  (Methodological)}, pages 289--300, 1995.

\bibitem{bolton2002statistical}
Richard~J Bolton and David~J Hand.
\newblock Statistical fraud detection: A review.
\newblock \emph{Statistical science}, pages 235--249, 2002.

\bibitem{cesabianchi2006}
Nicolo~Cesa-Bianchi and Gabor.~Lugosi.
\newblock \emph{{Prediction, Learning, and Games}}.
\newblock Cambridge University Press, 2006.

\bibitem{chandola2007outlier}
Varun Chandola, Arindam Banerjee, and Vipin Kumar.
\newblock Outlier detection: A survey.
\newblock \emph{ACM Computing Surveys}, 2007.

\bibitem{fawcett1997adaptive}
Tom Fawcett and Foster Provost.
\newblock Adaptive fraud detection.
\newblock \emph{Data mining and knowledge discovery}, 1\penalty0 (3):\penalty0
  291--316, 1997.

\bibitem{freund99}
Yoav~Freund and Robert E. Schapire.
\newblock Large margin classification using the perceptron algorithm.
\newblock \emph{Machine Learning}, 37\penalty0 (3):\penalty0 277--296, 1999.

\bibitem{hastie01}
Trevor~Hastie, Robert~Tibshirani, and Jerome~Friedman.
\newblock \emph{The Elements of Statistical Learning}.
\newblock Springer-Verlag, Heidelberg, 2001.

\bibitem{herbster2001tracking}
Mark~Herbster and Manfred K. Warmuth.
\newblock {Tracking the best linear predictor}.
\newblock \emph{The Journal of Machine Learning Research}, 1:\penalty0
  281--309, 2001.

\bibitem{kivinen1995perceptron}
Jyrki Kivinen and Manfred~K Warmuth.
\newblock The perceptron algorithm vs. winnow: linear vs. logarithmic mistake
  bounds when few input variables are relevant.
\newblock In \emph{Proceedings of the eighth annual conference on Computational
  learning theory}, pages 289--296. ACM, 1995.

\bibitem{kleinberg2003impossibility}
Jon Kleinberg.
\newblock An impossibility theorem for clustering.
\newblock \emph{Advances in neural information processing systems}, pages
  463--470, 2003.

\bibitem[Koprinska and Carrato(2001)]{koprinska2001temporal}
Irena Koprinska and Sergio Carrato.
\newblock Temporal video segmentation: A survey.
\newblock \emph{Signal processing: Image communication}, 16\penalty0
  (5):\penalty0 477--500, 2001.

\bibitem{liese2007statistical}
Friedrich Liese and Klaus-J Miescke.
\newblock Statistical decision theory.
\newblock In \emph{Statistical Decision Theory}, pages 1--52. Springer, 2007.

\bibitem{ljung1990}
Lennart~Ljung and Svante~Gunnarsson.
\newblock {Adaptation and tracking in system identification - A survey}.
\newblock \emph{Automatica (Journal of IFAC)}, 26\penalty0 (1):\penalty0 7--21,
  1990.

\bibitem{nichols2012multiple}
Thomas~E. Nichols.
\newblock Multiple testing corrections, nonparametric methods, and random field
  theory.
\newblock \emph{Neuroimage}, 62\penalty0 (2):\penalty0 811--815, 2012.
\newblock \doi{10.1016/j.neuroimage.2012.04.014}.

\bibitem{robbins1951stochastic}
Herbert~Robbins and Steven~Monro.
\newblock {A stochastic approximation method}.
\newblock \emph{The Annals of Mathematical Statistics}, pages 400--407, 1951.

\bibitem{rousseeuw2005robust}
Peter~J Rousseeuw and Annick~M Leroy.
\newblock \emph{Robust regression and outlier detection}, volume 589.
\newblock John Wiley \& Sons, 2005.

\bibitem{scholkopf2001estimating}
Bernhard Sch{\"o}lkopf, John~C Platt, John Shawe-Taylor, Alex~J Smola, and
  Robert~C Williamson.
\newblock Estimating the support of a high-dimensional distribution.
\newblock \emph{Neural computation}, 13\penalty0 (7):\penalty0 1443--1471,
  2001.

\bibitem{shalev2014understanding}
Shai Shalev-Shwartz and Shai Ben-David.
\newblock \emph{Understanding machine learning: From theory to algorithms}.
\newblock Cambridge university press, 2014.

\bibitem{smola1998learning}
Alex~J Smola and Bernhard Sch{\"o}lkopf.
\newblock \emph{Learning with kernels}.
\newblock Citeseer, 1998.

\bibitem{tsai2009intrusion}
Chih-Fong Tsai, Yu-Feng Hsu, Chia-Ying Lin, and Wei-Yang Lin.
\newblock Intrusion detection by machine learning: A review.
\newblock \emph{Expert Systems with Applications}, 36\penalty0 (10):\penalty0
  11994--12000, 2009.

\bibitem{van2004detection}
Harry~L Van~Trees.
\newblock \emph{Detection, estimation, and modulation theory}.
\newblock Wiley. com, 2004.

\bibitem{vapnik95}
Vladimir~Vapnik.
\newblock \emph{The Nature of Statistical Learning Theory}.
\newblock Springer-Verlag, New York, 1995.

\bibitem{zadeh2009uniqueness}
Reza~Bosagh Zadeh and Shai Ben-David.
\newblock A uniqueness theorem for clustering.
\newblock In \emph{Proceedings of the twenty-fifth conference on uncertainty in
  artificial intelligence}, pages 639--646. AUAI Press, 2009.

\end{thebibliography}

%

%\begin{IEEEbiography}[{\includegraphics[width=1in,height=1.25in,clip,keepaspectratio]{kp7.pdf}}]
%{Kristiaan Pelckmans}
%% or if you just want to reserve a space for a photo:
%%\begin{IEEEbiography}{Kristiaan Pelckmans}
%	Kristiaan Pelckmans is an assistant professor/senior researcher at Uppsala University (UU), 
%	Department of Information Technology (IT), Division of systems and control (SYSCON).
%	His main area of research is found on the crossroad between machine learning, 
%	automatic control and statistics. He has also been publishing in areas 
%	as diverse as medical (cancer) analysis, neuroimaging, information theory, learning theory, 
% 	stability of control systems, astrophysics, software engineering, evolutionary biology, and others. 
% 	Those applications have a common denominator in that there is an aspect of learning,
% 	linear algebra or convex optimization involved.
% 	Currently, he is teaching an undergraduate course on system identification.
% 	He has been collecting experience in the technical, data-based sciences since his Ph.D. 
% 	degree in 2005. He is funded for his research recently by grants awarded by 
% 	VetenskapsRadet (VR, SE), and the Knut \& Alice Wallenbergs (KAW,SE) stiftelse.
%	
%\end{IEEEbiography}
%

\end{document}